\documentclass{article}
\PassOptionsToPackage{numbers, compress}{natbib}
 \usepackage[preprint]{neurips_2026}

\usepackage[utf8]{inputenc}
\usepackage[T1]{fontenc}
\usepackage{url}
\usepackage{microtype}
\usepackage{graphicx}
\usepackage{subcaption}
\usepackage{booktabs}
\usepackage{nicefrac}
\usepackage{xcolor}
\usepackage{hyperref}
\usepackage{algorithm}
\usepackage{algorithmic}

\usepackage{Akbar}

\usepackage{amsmath}
\usepackage{amssymb}
\usepackage{mathtools}
\usepackage{amsthm}

\usepackage[capitalize,noabbrev]{cleveref}

\theoremstyle{plain}
\newtheorem{theorem}{Theorem}[section]

\newtheorem{corollary}[theorem]{Corollary}
\theoremstyle{definition}
\newtheorem{definition}[theorem]{Definition}

\theoremstyle{remark}
\newtheorem{remark}[theorem]{Remark}

\usepackage[textsize=tiny]{todonotes}

\usepackage{bm}
\usepackage{enumitem}
\usepackage{multirow}
\usepackage{tabularx}

\usepackage{placeins}
\usepackage{float}

\newcommand{\R}{\mathbb{R}}

\newcommand{\one}{\mathbf{1}}

\newcommand{\argmin}{\operatorname*{arg\,min}}
\newcommand{\argmax}{\operatorname*{arg\,max}}
\newcommand{\conv}{\operatorname{conv}}

\newcommand{\ext}{\operatorname{ext}}

\newcommand{\diam}{\mathrm{diam}}

\definecolor{Red}{RGB}{244, 124, 124}
\definecolor{Green}{RGB}{162, 222, 147}
\definecolor{Blue}{RGB}{112, 161, 211}
\newcommand{\rkone}[1]{{\setlength{\fboxsep}{1.5pt}\colorbox{Red!50}{#1}}}
\newcommand{\rktwo}[1]{{\setlength{\fboxsep}{1.5pt}\colorbox{Blue!50}{#1}}}

\title{Self-Supervised Combinatorial Optimization with Constraints via Frank--Wolfe}

\author{Akbar Rafiey\thanks{Equal contribution.} \\
  NYU\\
  \texttt{ar9530@nyu.edu} \\
  \And
  Yifei Xu\footnotemark[1]\\
  NYU \\
  \texttt{yx3590@nyu.edu} \\
  \And
  Nikolaos Karalias \\
  MIT\\
  \texttt{stalence@mit.edu} }

\begin{document}

\maketitle

\begin{abstract}
Self-supervised learning for combinatorial optimization has emerged as a promising paradigm for solving discrete optimization problems with neural networks, but a central challenge remains: handling hard combinatorial constraints within continuous, gradient-based training. Continuously extending combinatorial objectives to convex domains is a powerful technique, yet existing approaches often require projection steps that constrain neural network outputs to lie inside the feasible polytope and rely on ad-hoc and problem-specific constructions. We propose a general framework in which the neural network is allowed to predict arbitrary continuous vectors that could potentially lie outside of the feasible polytope. These predictions are then approximated by sparse convex combinations of feasible solutions using a geometric decomposition algorithm based on Frank–Wolfe methods and approximate Carath\'eodory results. This decomposition induces an a.e.-differentiable, self-supervised loss defined as the expected value of the discrete objective. The same procedure provides an automatic rounding guarantee at inference time. We demonstrate strong empirical performance across multiple combinatorial problems, including the Quadratic Assignment Problem, Maximum Coverage, and the Traveling Salesperson Problem.
\end{abstract}

\section{Introduction}
\label{sec:intro}

Combinatorial optimization (CO) forms a central pillar of optimization theory and practice, encompassing a broad class of problems in which one seeks to optimize an objective over a discrete, often exponentially large, feasible set. Such problems arise ubiquitously across science and engineering, from routing and resource allocation to learning and inference. Canonical examples include the Traveling Salesperson Problem (TSP), Maximum Coverage (MC), and Quadratic Assignment Problem (QAP), among many others. What unifies these problems is not only their computational hardness, but also the rich algorithmic structure induced by discrete constraints that describe large-scale discrete combinatorial spaces of configurations.

Despite this shared structure, successful approaches to CO have historically been case-specific, often reflecting deep problem-dependent insights; "stroke-of-genius" heuristics. A common algorithmic template, especially in approximation algorithms, relaxes the discrete problem into a continuous or convex surrogate, solves for a fractional solution, and then applies a bespoke rounding or improvement procedure to recover feasibility. While this relax-and-round paradigm has led to remarkable practical results and theoretical guarantees, adapting this approach to  gradient-based data-driven settings can be challenging. This limitation and challenge is especially pronounced when problem instances are drawn from structured distributions, increasingly common in real world applications. Self-supervised learning (SSL) approaches for combinatorial optimization offer an appealing alternative: unlike static approximation algorithms and heuristics, they can leverage latent patterns in data, train efficiently without requiring large amounts of labeled solutions, and incorporate algorithmic priors that facilitate learning and generalization.

In this work, we bridge the gap between discrete optimization and deep learning by building on work that used continuous extensions of discrete functions as losses for neural CO \cite{Karalias2022NeuralSF,karalias2025geometric,nerem2025differentiable}. The central idea in this line of work is to use extensions to embed discrete constraints and objectives directly into a learning pipeline, allowing learning and optimization to proceed end-to-end. By smoothing the discrete landscape, these approaches enable an unsupervised learning paradigm where models learn to exploit  patterns and discover high-quality solutions without requiring expensive labels.

More specifically, we propose a generic learning-based pipeline that integrates hard combinatorial constraints directly into training and inference. The neural network is allowed to output an arbitrary continuous vector, without any feasibility requirement. This output is then passed through a geometric decomposition algorithm inspired by approximate Carath\'eodory and Frank–Wolfe methods \cite{barman2015approximating,mirrokni_tight_nodate,combettes_revisiting_2023}, which produces a sparse convex combination of feasible solutions. The resulting decomposition may differ substantially from the network output but is always expressed entirely in terms of feasible points. This effectively shifts the burden of constraint enforcement from tuning penalty terms to capturing constraints algorithmically through our decomposition algorithm. The decomposition induces a (a.e. differentiable) distribution over feasible solutions, and the expected value of the discrete objective under this distribution is used directly as a self-supervised training loss. In addition, the discrepancy between the network output and its decomposition can be quantified and included as an auxiliary regularization term, with an associated weight. While not required by the framework, we find empirically that including this term improves optimization stability and solution quality. Gradients propagate through this objective using automatic differentiation. At inference time, the same procedure yields a small set of feasible candidate solutions from which the best one is selected. Figure \ref{fig:pipeline} summarizes the pipeline.

Our pipeline is projection-free. The decomposition in \cite{karalias2025geometric} requires the neural prediction to lie inside the feasible polytope, with feasibility ensured separately through problem-specific constructions. For example, it uses distinct projections for the hypersimplex (Proposition 4.4) and the spanning-tree polytope (Theorem E.7). Our method removes the need for projection and accepts arbitrary ambient-space outputs and uses a Frank--Wolfe decomposition to construct a sparse convex combination of feasible vertices. 

Moreover, our decomposition has an explicit run time certificate and is more efficient in terms of the number of Linear Maximization Oracle (LMO) calls. It requires exactly one exact or approximate LMO call per iteration, so $T$ iterations use exactly $T$ LMO calls. In contrast, each iteration of the generic GLS (Groetschel/Lovasz/Schrijver) decomposition in \cite{karalias2025geometric} invokes minimal-face and ray-boundary optimization routines that themselves require solving convex optimization problems using the oracle; see the proof of Theorem 4.1. Consequently, a single GLS iteration may require many LMO calls. For example, a standard ellipsoid-based implementation requires approximately $O(n^2L)$ oracle calls per iteration and $O(n^3L)$ calls overall, where $L$ is a precision parameter.

Although specialized GLS decompositions can be more efficient for particular polytopes, deriving them requires nontrivial, problem-specific techniques. Given an efficient exact or approximate LMO, our procedure applies unchanged across different polytopes, with only the LMO changing. It thus provides a unified oracle-based method for hypersimplex, Birkhoff, spanning-tree, and matching polytopes, subsuming the specialized settings considered in \cite{karalias2025geometric,nerem2025differentiable}.

\begin{figure*}[t]
     \centering    
\includegraphics[width=0.95\textwidth]{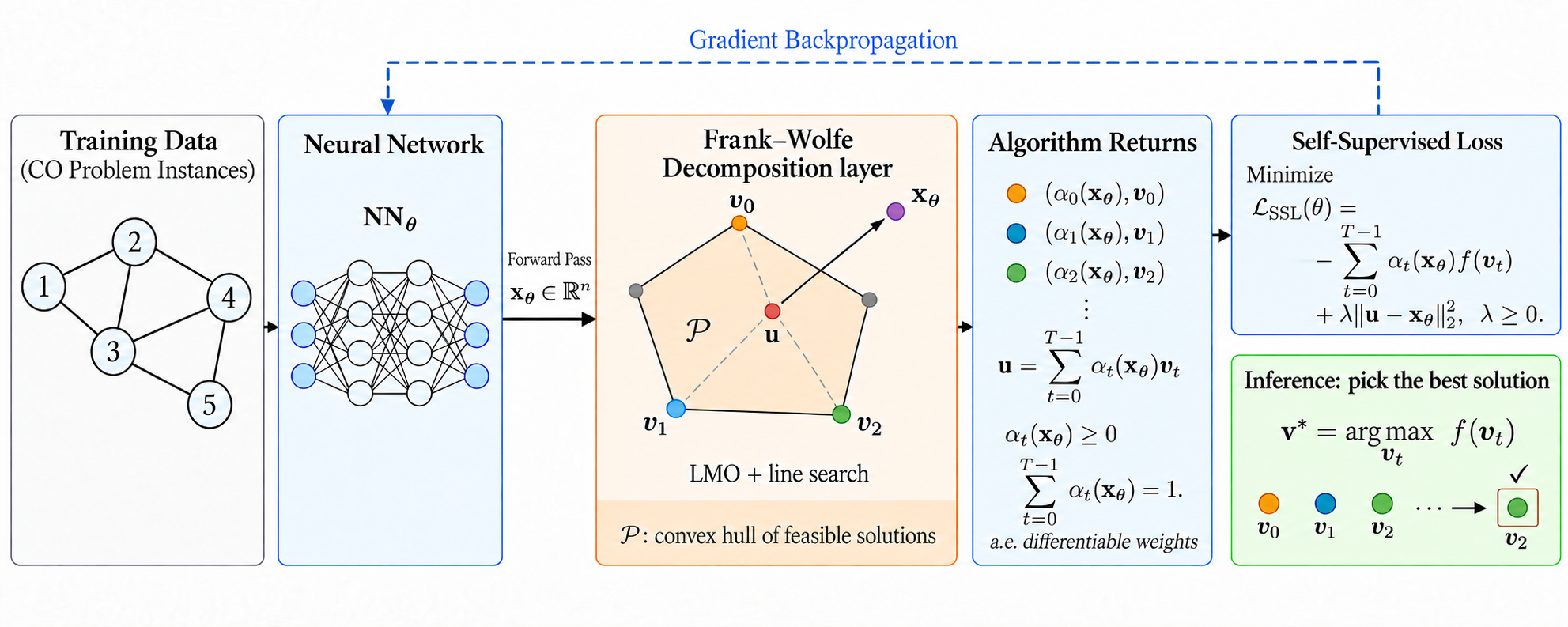}
     \caption{Overview of our framework. During training, the model with parameters $\theta$ outputs a point $\bx_\theta$, which may lie outside $\mathcal{P}$, the convex hull of feasible solutions. Our Frank--Wolfe decomposition layer constructs a feasible proxy $\bu$ as a convex combination of at most $T$ feasible vertices $\bv_t$, with weights $\alpha_t(\bx_\theta)$. The self-supervised loss $\mathcal{L}_{\mathrm{SSL}}$ is computed from the expected discrete objective $f$ and a reconstruction term weighted by $\lambda\geq 0$. The weights $\alpha_t(\bx_\theta)$ are a.e. differentiable w.r.t $\bx_\theta$, allowing gradients from the loss to propagate through the decomposition layer back to the neural network. During inference (green box), the same procedure generates feasible candidates, and the best solution $\bv^*$ is selected. Solid arrows indicate the forward pass, and the dashed arrow indicates backpropagation.}
    \label{fig:pipeline}
\end{figure*}

\paragraph{Contributions.} We make the following contributions:

\begin{itemize}[leftmargin=*]
\item \textbf{A projection-free, self-supervised framework for CO.}
We introduce a unified learning framework that enables neural networks to optimize combinatorial objectives under hard discrete constraints by operating over convex polytopes of feasible solutions. Unlike prior geometric-extension methods, our framework does \emph{not} require network outputs to lie in the polytope during training and provides a common formulation across a broad class of combinatorial structures.

\item \textbf{An oracle-based extension layer using Frank--Wolfe decomposition.}
To realize this framework, we propose a Frank--Wolfe–style extension layer that maps an arbitrary continuous prediction to a sparse convex combination of feasible solutions, yielding a distribution over feasible discrete solutions. Given an efficient exact or approximate linear maximization oracle (LMO), the same procedure applies across different polytopes; only the implementation of the LMO changes. It requires one LMO call per decomposition iteration and avoids problem-specific projections and specialized decomposition algorithms.


\item \textbf{A differentiable, self-supervised loss with built-in rounding guarantees.}
The induced convex decomposition allows us to define a differentiable loss as the expected discrete objective, with approximation error explicitly incorporated. At inference time, selecting the best solution in the support yields a feasible solution whose objective value is no worse than the expectation. This provides a direct rounding guarantee and unifies learning and solution recovery within a single pipeline.


\item \textbf{Empirical validation across diverse combinatorial domains.} We demonstrate the effectiveness of our framework on a range of classical CO problems—including MC, TSP, and QAP—showing strong empirical performance and broad applicability across diverse constraint families.
\end{itemize}

\section{Related Work}
\label{sec:related}

\textbf{Extensions and optimization.} A long line of work in CO views discrete problems through the lens of continuous optimization by embedding feasible solutions into a convex polytope and optimizing over this relaxation. This perspective is deeply rooted in algorithm design and polyhedral geometry, where constructing convex or concave extensions with favorable optimization properties has been a central theme \cite{lovasz1983submodular,crama1993concave,murota1998discrete,tawarmalani2002convex}. In particular, the convex closure provides the tightest convex extension of a discrete set function and underlies many relaxation-based methods \cite{falk1976successive,tawarmalani2002convex}. Classical solvers and approximation algorithms routinely exploit this framework: prominent examples include LP and SDP relaxations for problems such as Max Cut and TSP, with solvers like Concorde and Gurobi relying on large-scale LP and cutting-plane techniques \cite{goemans1995improved,williamson2011design,dantzig1954solution,applegate2011traveling}. While these approaches yield strong theoretical guarantees, they typically rely on problem-specific, rounding procedures that are non-differentiable and difficult to integrate into modern, data-driven learning pipelines.

\textbf{Neural CO.}
Recent work has explored unsupervised and self-supervised learning approaches for CO that optimize continuous surrogates of discrete objectives, typically induced through probabilistic relaxations and expected-value formulations \cite{amizadeh2018learning,karalias2020erdos,toenshoff2021graph,wang2022unsupervised}. Compared to reinforcement learning or supervised methods \cite{khalil2017learning,kool2019attention,vinyals2015pointer}, these approaches avoid labeled data and often exhibit more stable training.
Unsupervised CO pipelines involve several crucial components, including neural network architecture \citep{min2022can,vinyals2015pointer,sato2019approximation} and the role of input features for well-known classes of models \citep{lim2023sign,keriven2023functions}. In this work, our focus is on loss function design and rounding. A common strategy in self-supervised CO is to parameterize a distribution over discrete solutions and use the expected value of the discrete objective as a differentiable training loss \cite{karalias2020erdos,wang2022unsupervised,bu2024tackling}, enabling smoother optimization but raising challenges related to constraint enforcement.
A major consideration in these approaches is enforcing constraints on the outputs of neural networks. Techniques in the literature typically involve projecting the network output onto the feasible set. This has been explored for cardinality constraints using ideas from optimal transport \citep{wang2023linsatnet}, and for general linear constraints using gradient-based methods \citep{donti2021dc3,zeng2024glinsat}. This also includes Sinkhorn-style projection methods and their extensions \citep{sinkhorn1964relationship,wang2022towards}. In the continuous (and potentially non-convex) optimization setting, several projection techniques have also been developed \citep{liang2023low,liang2024homeomorphic,liangefficient,li2025gauge}. Other approaches add penalty terms or improve training dynamics through annealing or physics-inspired formulations \cite{sun2022annealed,schuetz2022combinatorial,mundinger2025neural}.
Our method differs from these approaches by incorporating feasibility directly through geometric decomposition. We define a distribution supported entirely on feasible solutions, allowing constraints to be built into the distribution itself. Thus, the training loss focuses solely on the discrete objective while remaining a.e. differentiable. A closely related line of work develops neural CO methods based on geometric relaxations of feasibility polytopes. \citet{Karalias2022NeuralSF,karalias2025geometric} propose differentiable extensions for set-function optimization under various constraints, while \citet{nerem2025differentiable} study CO problems over permutations. These approaches typically require the neural network output to lie within the feasibility polytope. In contrast, our approach allows arbitrary continuous network outputs and applies a general-purpose geometric decomposition across a broad class of combinatorial constraints.

\textbf{Approximate Carath\'eodory and Frank-Wolfe.}
Our approach builds on results from convex geometry and projection-free optimization. A central component is the Frank--Wolfe algorithm \cite{frank1956algorithm, canon_tight_1968}, a first-order method for constrained optimization. Several variants with improved convergence have been studied \cite{jaggi_revisiting_nodate, lacoste2015global}, and Frank--Wolfe has been widely applied to structured and combinatorial polytopes \cite{garber2016linear, lacoste2013block, krishnan2015barrier}. Approximate variants of Frank--Wolfe with provable guarantees have also been proposed \cite{locatello2017unified,zhou_approximate_2022}.
From a geometric perspective, the Carath\'eodory theorem \cite{caratheodory1907} states that any point in a polytope $P \subseteq \mathbb{R}^d$ can be expressed as a convex combination of at most $d+1$ corners of the polytope. Approximate versions provide dimension-independent bounds on sparse approximations \cite{barman2015approximating, mirrokni_tight_nodate}. Obtaining such sparse convex combinations of polytope vertices is essential for our framework, and Frank--Wolfe-type algorithms provide a natural algorithmic mechanism to achieve this \cite{mirrokni_tight_nodate, combettes_revisiting_2023}.



\section{Problem Formulation and Learning Setup}

We consider CO problems defined over a discrete feasible set. Each problem instance is specified by a collection of variables $V = \{v_1, \dots, v_n\}$, each assigned a value from a discrete domain $D=\{1,\dots,d\}$, a finite set of feasible solutions $\mathcal{C}\subseteq D^n$, and a real-valued objective function $f : \mathcal{C} \to \mathbb{R}$. The goal is to find an optimal feasible solution $\bx^* \in \mathcal{C}$ that optimizes (either minimizes or maximizes) the 
objective function:
\begin{align}
\label{eq:main-CO}
\max f(\bx), \quad \text{s.t.} \quad \bx\in \mc{C}.
\end{align}
This abstract formulation accommodates a broad class of CO problems, including those defined over binary vectors, general discrete assignments, and permutations.

Many examples in this paper focus on problems with Boolean decision variables, where solutions can be represented by indicator vectors encoding subset or structure selection under constraints. Two representative running examples are Maximum Coverage, which selects a subset of elements subject to a budget constraint, and the TSP, where binary variables indicate whether edges are included in a Hamiltonian tour. Our framework also applies to CO problems over permutations, e.g., the QAP, and to non-Boolean discrete domains, e.g., $k$-submodular maximization.

A standard geometric perspective on CO is to consider the convex hull of all feasible solutions,
\[
\mathcal{P} = \mathrm{conv}(\mathcal{C}).
\]
For CO problems with Boolean domain, we have $\mathcal{P} = \mathrm{conv}(\mathcal{C})\subseteq [0,1]^n$. As an example, when $\mathcal{C}$ consists of all subsets of size $k$, $\mathcal{P}$ is the associated cardinality polytope, a.k.a hypersimplex. Each $S \in \mathcal{C}$ is represented by its indicator vector $\one_S \in \{0,1\}^n$, where $(\one_S)_i = 1$ if element $i$ belongs to $S$ and $(\one_S)_i = 0$ otherwise. The polytope associated with $\mathcal{C}$ is then $
  \mathcal{P} := \conv\{\one_S : S \in \mathcal{C}\} \subseteq [0,1]^n,$
whose vertices, i.e., extreme points $\ext(\mathcal{P})$, coincide with the indicator vectors of subsets of size $k$.

Note that although $\mathcal{P}$ may have an exponentially large facet description, many important combinatorial polytopes admit efficient linear optimization via polynomial-time separation or optimization oracles.



\subsection{Self-Supervised Learning for CO}

In the self-supervised learning setting, each problem instance $\mathcal{I}$ is described by input features $\mathbf{Z}_\mc{I}$ possibly together with instance-specific structure. These features are processed by a neural network $\mathrm{NN}_\theta$, which outputs a representation $\bx_\theta$. Importantly, this output is not required to lie in the convex polytope of feasible solutions. 

Rather than training the network to imitate optimal solutions, learning is guided by a loss function defined on an a.e. differentiable extension of the combinatorial objective. This extension is constructed using the geometry of the polytope $\mathcal{P}$, allowing the objective to be meaningfully evaluated on fractional solutions and optimized using gradient-based methods. 

The network is trained over a collection of instances without access to ground-truth optimal solutions, making the design of the loss function central. At inference time, the continuous output $\bx$ is converted into a feasible discrete solution using a principled, geometry-based procedure, rather than heuristic rounding. Moreover, because no labeled data is required, the same objective can be optimized directly on test instances, allowing additional computation at test time to further improve solution quality.

\section{Proposed Method}

\label{sec:background}

We now describe the overall pipeline. Consider a polytope $\mc{P}=\mathrm{conv}(\mathcal{C})$, where $\mc{C}\subset \zR^n$ denotes the set of feasible solutions for a CO problem. Let $\ext(\mathcal{P})\subseteq\mc{C}$ denote the set of extreme points i.e., corners, of $\mc{P}$ which correspond to feasible combinatorial solutions.

Given a target point $\bx \in \mathbb{R}^n$, which can be thought of as the output of a neural network, our goal is to find a point $\bu \in \mathcal{P}$ that serves as a good proxy for $\bx$ and satisfies desirable properties. In particular, we want $\bu$ to be a \emph{sparse} convex combination of $T$ points $\mathbf{v}_0, \dots, \mathbf{v}_{T-1} \in \ext(\mathcal{P})$. That is, for a small $T$, we have
\begin{align}
\label{eq:property}
    \bu(\bx) = \sum_{t=0}^{T-1} \alpha_t(\bx) \, \bv_t \in \mathcal{P} 
    \quad \text{with} \quad
    \begin{cases}
    \bv_t \in \ext(\mathcal{P}), \\
    \sum_{t=0}^{T-1} \alpha_t(\bx) = 1;~ 
    \alpha_t \ge 0.
    \end{cases}
\end{align}

The coefficients $\alpha_t(\bx)$ naturally define a distribution $\mathcal{D}(\bx)$ over feasible solutions. We define the FW-induced objective
\begin{align}
    F(\bx) = \sum_{t=0}^{T-1} \alpha_t(\bx) \, f(\bv_t)
           = \mathbb{E}_{\bv \sim \mathcal{D}(\bx)}[f(\bv)].
\end{align}


\noindent\textbf{Training.} In the learning pipeline, let $\bx_\theta$ denote the prediction of
a neural network with parameters $\theta$. For a maximization problem \eqref{eq:main-CO},
we train the network by minimizing the self-supervised loss
\begin{align}
    \label{eq:loss}
    \mathcal{L}_{\mathrm{SSL}}(\theta):=-F(\bx_\theta)+\lambda\|\bx_\theta-\bu(\bx_\theta)\|_2^2 \, , \qquad \lambda \geq 0.
\end{align} 

In order to optimize the loss function using gradient-based methods, the weights $\alpha_t(\bx_\theta)$ must be differentiable functions of $\bx_\theta$. Under this condition, derivatives of the loss with respect to the network parameters can be computed via the chain rule:
$
\frac{\partial F}{\partial \theta}
= \frac{\partial F}{\partial \boldsymbol{\alpha}(\bx_\theta)}
  \cdot \frac{\partial \boldsymbol{\alpha}(\bx_\theta)}{\partial \bx_\theta}
  \cdot \frac{\partial \bx_\theta}{\partial \theta},
$ where $\frac{\partial F}{\partial {\alpha}_t(\bx_\theta)}=f(\bv_t)$ and $\frac{\partial \boldsymbol{\alpha}(\bx_\theta)}{\partial \bx_\theta}$ comes from differentiating the steps our FW decomposition algorithm, Algorithm~\ref{alg:fw}. Note that under the stable-oracle decision assumption introduced in Section~\ref{sec:algfw}, the selected vertices $\mathbf{v}_t$ are locally constant functions of $\mathbf{x}_\theta$ almost everywhere. Hence, $\frac{\partial \mathbf{v}_t}{\partial \mathbf{x}_\theta}=0$ wherever the derivative exists, and the gradient propagates only through the coefficients. The gradient of the term $\lambda\|\bx_\theta-\bu(\bx_\theta)\|_2^2$ is computed similarly by differentiating through $\mathbf u(\bx_\theta)=\sum_{t=0}^{T-1}\alpha_t(\bx_\theta)\mathbf v_t$.

In Section \ref{sec:algfw}, we present our algorithm, which given $\bx$ computes $\bu(\bx) \in \mathcal{P}$ satisfying the properties in \eqref{eq:property}. Moreover, we show that the weights $\alpha_t(\bx)$ returned by our algorithm are a.e. differentiable w.r.t. $\bx$ and provide a convergence rate in $T$, the number of iterations of the algorithm. For notational simplicity, throughout Section \ref{sec:algfw} we suppress the dependence on the target point $\bx$, writing $\bu$ and $\alpha_t$ in place of $\bu(\bx)$ and $\alpha_t(\bx)$.

\begin{remark}[Extension property]
The FW-induced objective $F$ defines a surrogate on $\mathbb{R}^n$. For $F$ to be an extension of $f:\mathcal C\to\mathbb R$, it must agree with $f$ on $\mathcal C$. We therefore adopt the convention that Algorithm~\ref{alg:fw} returns the singleton decomposition ${(1,\bx)}$ whenever $\bx\in\mathcal C$. Consequently, $F(\mathbf c)=f(\mathbf c)$ for every $c\in\mathcal C$. Since $\mathcal C$ is finite, this convention modifies $F$ only on a set of Lebesgue measure zero and therefore does not affect its almost-everywhere differentiability.
\end{remark}

\subsection{From NN Output to Polytope: Frank-Wolfe Decomposition}
\label{sec:algfw}

In this section, we describe our algorithm presented in \cref{alg:fw}. Given $\bx$, define the following quadratic objective:
\begin{equation}
\label{eq:J}
  J(\mathbf{u}) = \tfrac{1}{2} \|\mathbf{x} - \mathbf{u}\|_2^2, \qquad \mathbf{u} \in \mathcal{P}.
\end{equation}
Starting at an initial point $\bu_0$ in the polytope our goal is to iteratively move towards $\bu_1,\bu_2,..,\bu_{T-1}$ so that at each iteration we decrease $J(\bu_t)$, and at the end $\bu_{T-1}$ has the desired aforementioned properties. More specifically, we start with an initial point $\bu_0$ which is a corner of the polytope $\mc{P}$, i.e. $\bv_0=\bu_0\in \mathrm{ext}(\mc{P})$. Then, in the first iteration, we find another corner of the polytope $\mc{P}$, say $\bv_1$ and set $\bu_{1}=(1-\gamma_1)\bu_{0}+\gamma_1 \bv_{1}$ with an appropriate value $\gamma_1\in[0,1]$, to be determined later. By construction, $\bu_1\in \mc{P}$ as it is a convex combination of two corners of the polytope. This iterative process is repeated and at every iteration $t$ we have
\begin{align}
    \bu_{t}=(1-\gamma_t)\bu_{t-1}+\gamma_t \bv_{t}
\end{align}
with an appropriate value for $\gamma_t\in[0,1]$ and $\bv_t\in \mathrm{ext}(\mc{P})$. The algorithm terminates after $T-1$ iterations and returns $\bu_{T-1}$, which will lie in $\mc{P}$ by construction. At this point, $\bu_{T-1}$ is a convex combination of $\bu_{0}=\bv_0,\dots,\bv_{T-1}$ with weights $\alpha_0,\dots,\alpha_{T-1}$ that are calculated recursively with respect to $\gamma_t$s. Specifically, $\sum_{t=0}^{T-1}\alpha_t=1$, with $\alpha_0=\prod_{i=1}^{T-1}(1-\gamma_i)$ and $\alpha_t= \gamma_t \prod_{i=t+1}^{T-1}(1-\gamma_i)$.


\begin{algorithm}[t]
\caption{\makebox[\linewidth][l]{\textsc{FW Decomposition} $(\mathbf{x},\mathcal{P},T;\mathrm{LMO})$}}
\label{alg:fw}
\begin{algorithmic}[1]
\STATE \textbf{Input:} Target $\mathbf{x} \in \mathbb{R}^n$, polytope $\mathcal{P}$, support budget $T$, and an exact or approximate LMO.
\STATE \textbf{Initialize:} $\mathbf{v}_0=\mathbf{u}_0 =\argmax_{\bv\in\mc{P}}\langle \mathbf{x}, \mathbf{v} \rangle$, $\alpha_0 \gets 1$.
\FOR{$t=1$ to $T-1$}
  \STATE $\mathbf{r}_t \gets \mathbf{x} - \mathbf{u}_{t-1}$ \COMMENT{Residual}
  \STATE $\mathbf{v}_t \gets \argmax_{\bv\in\mc{P}}\langle \mathbf{r}_t, \mathbf{v} \rangle$ \COMMENT{Find a corner via LMO}
  \STATE $\mathbf{d}_t \gets \mathbf{v}_t - \mathbf{u}_{t-1}$ \COMMENT{Descent direction}
  \STATE $\gamma_t \gets \mathrm{Clip}(0, 1, \langle \mathbf{r}_t, \mathbf{d}_t \rangle / \|\mathbf{d}_t\|_2^2)$ \COMMENT{Line search step}
  \STATE $\alpha_j \gets (1 - \gamma_t) \alpha_j$ for all existing weight $\alpha_j$, $j \in \{0, \dots, t-1\}$. 
  
  \STATE $\alpha_{t} \gets \gamma_t$
  \STATE $\mathbf{u}_{t} \gets (1 - \gamma_t)\mathbf{u}_{t-1} + \gamma_t \mathbf{v}_t$
\ENDFOR
\STATE \textbf{Output:} Weights and vertices $\{(\alpha_t, \mathbf{v}_t)\}_{t=0}^{T-1}$ such that $\mathbf{u}_{T-1} = \sum_{t=0}^{T-1} \alpha_t \mathbf{v}_t$.
\end{algorithmic}
\end{algorithm}

The main two points to address here are which corner $\bv_t$ to choose and what is a good value for step size $\gamma_t$. 

Note that the function $J(\bu)$ is convex and its gradient is $\nabla J(\mathbf{u}) = \mathbf{u} - \mathbf{x}$. We define the \textbf{residual} at iteration $t$ as $\mathbf{r}_t := \mathbf{x} - \mathbf{u}_{t-1} = -\nabla J(\mathbf{u}_{t-1})$, which points from our current estimate toward the target. Frank--Wolfe is a "projection-free" method. At every iteration, instead of moving toward the target in the gradient direction and then projecting back into $\mathcal{P}$ (which can be computationally expensive), it finds the vertex $\mathbf{v}_t$ of the polytope that lies furthest in the direction of the residual $\mathbf{r}_t$. It then takes a small step $\gamma_t$ from the current point $\mathbf{u}_{t-1}$ toward that vertex $\mathbf{v}_t$. Since both $\mathbf{u}_{t-1}$ and $\mathbf{v}_t$ are in $\mathcal{P}$, their convex combination is guaranteed to remain inside $\mathcal{P}$. The exact  procedure is shown in \cref{alg:fw}, and in order to find $\mathbf{v}_t$ we require what is known as LMO.

\begin{definition}[Linear Maximization Oracle (LMO)]
Given a residual $\mathbf{r} \in \mathbb{R}^n$, an exact LMO returns a vertex $\mathbf{v}^\star$ that maximizes the alignment with $\mathbf{r}$:
\begin{align}
\label{eq:LMO-corner}
  \mathbf{v}^\star = \argmax\langle \mathbf{r}, \mathbf{v} \rangle \quad \text{s.t.} \quad \mathbf{v} \in \mathcal{P}
\end{align}
\end{definition}

A fundamental property of linear optimization is that a linear objective over a polytope attains an optimum at a vertex. Hence, an LMO returns a (integral) vertex of the polytope. When exact maximization is difficult, the convergence statement below uses a relative Frank--Wolfe gap approximation \cite{jaggi_revisiting_nodate,lacoste2013block,locatello2017unified,zhou_approximate_2022}. Formally, at iteration $t$, a $\delta$-relative FW-gap oracle returns $\bv_t \in \mathrm{ext}(\mc{P})$ satisfying
\begin{align}
\label{eq:relative-fw-gap}
\langle \br_t,\bv_t-\bu_{t-1}\rangle
\;\ge\;
\delta
\max_{\bv \in \mc{P}} \langle \br_t,\bv-\bu_{t-1}\rangle,
\end{align}
where $\br_t=\bx-\bu_{t-1}$ and $\delta\in(0,1]$. This is the approximation notion used in our convergence statement. Exact LMOs satisfy \eqref{eq:relative-fw-gap} with $\delta=1$.

For the step size $\gamma_t$, several choices are possible. We compute it using a simple line-search:
$$\gamma_t=\mathrm{Clip}(0, 1, \langle \mathbf{r}_t, \mathbf{d}_t \rangle / \|\mathbf{d}_t\|_2^2),$$
with the convention $\gamma_t=0$ when $\mathbf{d}_t=0$. This allows for efficient computation and supports branchwise differentiation of the weights $\{\alpha_t\}$.

\noindent\textbf{Assumption (stable oracle decisions).}
Fix the support budget $T$. We assume that the deterministic oracle and tie-breaking rule do not switch pathologically. The set of inputs $\bx$ at which an arbitrarily small perturbation can change one of the selected vertices, or change whether a line-search step is clipped, has zero volume. Equivalently, for almost every input $\bx$, there is a small neighborhood around $\bx$ on which Algorithm~\ref{alg:fw} selects the same vertices $\bv_0,\ldots,\bv_{T-1}$ and uses the same line-search clipping cases. This assumption rules out artificial rules that switch between equally valid approximate vertices on dense sets. It is satisfied by the standard comparison-based oracles used in combinatorial optimization, including top-$k$ selection, matroid greedy algorithms such as Kruskal’s algorithm, and exact assignment or matching oracles with fixed tie-breaking. 

All these results together yield the following theorem. The convergence guarantee is the standard exact Frank--Wolfe rate, and the a.e. differentiability proof is given in \cref{sec:proof}.

\begin{theorem}[Exact Frank--Wolfe Decomposition]
\label{thm:FWextension}
Let $\bx\in\zR^n$, and let $\mathcal{P}\subset\mathbb{R}^n$ be a  polytope with diameter
$D := \mathrm{diam}(\mathcal{P}) = \max_{\bu,\bv\in\mathcal{P}}\|\bu-\bv\|_2$.
Run Algorithm~\ref{alg:fw} with an exact LMO and support budget $T\ge 2$, so that it performs $T-1$ Frank--Wolfe updates. Then the algorithm outputs $\mathbf{u}_{T-1} = \sum_{t=0}^{T-1} \alpha_t \mathbf{v}_t \in \mathcal{P}$, where $\mathbf{v}_t\in\ext(\mathcal{P})$, $\alpha_t\ge0$, and $\sum_{t=0}^{T-1}\alpha_t=1$. Moreover, if $\mathbf{u}^\star=\argmin_{\bu\in\mathcal{P}}J(\bu)$, then
\[
J(\mathbf{u}_{T-1})-J(\mathbf{u}^\star)
\;\le\;
\frac{2D^2}{T+1}.
\]
The coefficients $\alpha_t$ are almost everywhere differentiable functions of $\mathbf{x}$.
\end{theorem}

\begin{corollary}[Approximate Frank--Wolfe Decomposition]
\label{cor:approx-fw-extension}
If the exact LMO in \cref{thm:FWextension} is replaced by a $\delta$-relative FW-gap oracle satisfying \eqref{eq:relative-fw-gap}, then the same convex-combination conclusion holds. Let $D=\mathrm{diam}(\mathcal P)$, $\mathbf{u}^\star=\argmin_{\bu\in\mathcal{P}}J(\bu)$, and $h_0=J(\bu_0)-J(\bu^\star)$. Then
\[
J(\mathbf{u}_{T-1})-J(\mathbf{u}^\star)
\;\le\;
\frac{2\left(D^2/\delta+h_0\right)}{\delta(T-1)+2}.
\]
The coefficient map remains almost everywhere differentiable.
\end{corollary}

\textbf{Discussion on projection onto the polytope.} In the above discussion we did not require neural network output to lie in the polytope, yielding a generic projection-free framework. Nevertheless, when this occurs---either by design of the neural network or because it is empirically advantageous---our framework aligns with \emph{approximate Carath\'eodory} results. When $\bx \in \mc{P}$, minimizing $J(\bu)=\tfrac12\|\bx-\bu\|_2^2$ via \cref{alg:fw} produces a sparse convex combination of vertices of $\mc{P}$ approximating $\bx$. Constructing such sparse convex combinations is the goal of approximate Carath\'eodory results. Unlike the exact Carath\'eodory theorem, the approximate variants relax exact representation in favor of an additive error to achieve sparser convex combinations \cite{pisier1981remarques, barman2015approximating, mirrokni_tight_nodate}. In the exact-LMO setting, \cref{thm:FWextension} yields the approximate Carath\'eodory bound
$\|\bx-\bu_{T-1}\|_2^2 \le \tfrac{4D^2}{T+1}.$

\subsection{Examples with Efficient (Approximate) LMO}
\label{sec:cases}

We briefly discuss how our framework extends to structured polytopes; details are in \cref{sec:cases-apendix}.

\noindent\textbf{Matroid polytopes.}
Matroid base polytopes capture many combinatorial constraints (e.g., cardinality, partition, spanning trees) and admit efficient LMOs via greedy maximum-weight base algorithms. Since our method requires only an efficient LMO, it applies uniformly across matroids without modification, unlike \cite{karalias2025geometric}, which requires constraint-specific projections and decompositions.

\noindent\textbf{Birkhoff and matching polytopes.}
The Birkhoff and (perfect) matching polytopes admit LMOs via maximum-weight matching, allowing direct optimization without projections or penalties. This is unlike prior work (e.g., \cite{nerem2025differentiable}), which relies on specialized decompositions and projection. Although exact LMOs may be costly at scale, we find that simple greedy approximations suffice empirically.

\section{Applications and Experiments}
\label{sec:experiments}
We evaluate our framework on three core CO problems—MC, QAP, and TSP—and show that the same decomposition and extension apply across all cases. Our method is competitive with and often outperforms SOTA approaches. Compared to SOTA approaches, which rely on problem-specific and novel engineered constructions, this generality is a key advantage. For testing methods and choice of baselines, we keep a strict one-shot setting (e.g., without Test Time Optimization (TTO)) to ensure fairness, detailed in \cref{sec:inference_settings}. Code and data are available in \href{https://anonymous.4open.science/r/FW_NCO-FA5C}{link}.

\subsection{Maximum Coverage}
\label{sec:exp-selection}
\begin{table*}[t]
\centering
\caption{MC performance comparison across three datasets, reporting mean inference time and coverage across $k$. Among neural methods, \textcolor{Red!50}{$\blacksquare$} indicates ranking the 1st, \textcolor{Blue!50}{$\blacksquare$} the 2nd in each column.}
\label{tab:max-coverage-results}
\setlength{\tabcolsep}{3pt} 
\resizebox{\textwidth}{!}{
\begin{tabular}{l cccc cccc cccc}
\toprule
& \multicolumn{4}{c}{\textbf{Random500}} & \multicolumn{4}{c}{\textbf{Random1000}} & \multicolumn{4}{c}{\textbf{Rail}} \\
\cmidrule(r){2-5} \cmidrule(lr){6-9} \cmidrule(l){10-13}
& \multicolumn{2}{c}{$k=10$} & \multicolumn{2}{c}{$k=50$} & \multicolumn{2}{c}{$k=20$} & \multicolumn{2}{c}{$k=100$} & \multicolumn{2}{c}{$k=20$} & \multicolumn{2}{c}{$k=50$} \\
\cmidrule(r){2-3} \cmidrule(r){4-5} \cmidrule(lr){6-7} \cmidrule(lr){8-9} \cmidrule(lr){10-11} \cmidrule(l){12-13}
\textbf{Method} & Time (s) & Coverage & Time (s) & Coverage & Time (s) & Coverage & Time (s) & Coverage & Time (s) & Coverage & Time (s) & Coverage \\
\midrule

Random (240s) & 240.00 & 13372.76 & 240.00 & 36786.89 & 240.00 & 24133.50 & 240.00 & 70527.31 & 240.00 & 5291.67 & 240.00 & 7367.00 \\
Gurobi (120s) & 11.574 & 15714.90 & 120.065 & 44880.59 & 37.298 & 31347.62 & 120.139 & 89696.83 & 121.059 & 5631.67 & 121.033 & 7604.67 \\
Greedy & 0.057 & 15640.99 & 0.121 & 44597.56 & 0.190 & 31105.89 & 0.460 & 88685.40 & 0.727 & 5617.00 & 1.320 & 7630.00 \\
Ucom2-short & 1.041 & 15253.35 & 0.739 & 44208.95 & 1.734 & 29791.66 & 1.621 & 88472.61 & 2.299 & 5512.67 & 2.454 & 7594.67 \\
\midrule
GeoNCO & 0.579 & 15177.77 & 1.108 & 41247.80 & 0.825 & 29810.12 & 1.908 & \rktwo{81357.29} & 1.856 & \rktwo{5343.00} & 2.084 & \rktwo{7411.67} \\
CardNN-noTTO-S & 1.688 & 9231.54 & 1.721 & 33055.87 & 1.869 & 18458.92 & 1.881 & 65793.40 & 1.929 & 5074.33 & 2.189 & 7193.00 \\
EGN-naive & 53.041 & \rkone{15262.76} & 120.136 & \rktwo{41272.68} & 120.010 & \rkone{29968.04} & 120.356 & 81166.12 & 120.676 & 5234.67 & 121.335 & 7408.33 \\
RL (GNN+Actor) & \rkone{0.069} & 14741.26 & \rkone{0.247} & 39510.96 & \rkone{0.102} & 29912.14 & \rkone{0.465} & 81158.54 & \rkone{0.219} & 5137.00 & \rkone{0.219} & \rkone{7456.67} \\
\midrule

\textbf{FWNCO (ours)} & \rktwo{0.426} & \rktwo{15192.64} & \rktwo{0.555} & \rkone{42172.44} & \rktwo{0.647} & \rktwo{29926.63} & \rktwo{0.913} & \rkone{83559.22} & \rktwo{1.432} & \rkone{5344.67} & \rktwo{1.931} & 7393.00 \\
\bottomrule
\end{tabular}
}
\end{table*}

Let $\mathcal{U}$ be a ground set of elements with weights $w_u \ge 0$ for $u\in\mathcal{U}$, and let $\{S_1,\dots,S_n\}$ be subsets $S_i \subseteq \mathcal{U}$. The value of $\bx\in\{0,1\}^n$ is the total weight of covered elements,
$
  f_{\text{cov}}(\bx)
  \;=\;
  \sum_{u \in \mathcal{U}} 
  w_u \,\mathbf{1}\!\left[\exists\, i \in [n]\ \text{with}\ u \in S_i \ \text{and}\ \bx(i) = 1\right].
$
The MC problem under cardinality constraint is $\max_{\bx \in \{0,1\}^n} f_{\text{cov}}(\bx)$ subject to $\sum_{i=1}^n \bx(i) = k$, where $k$ denotes the prescribed cardinality.

\textbf{Neural model and continuous targets.}
This problem can naturally be represented as a bipartite graph where the goal is to select $k$ nodes from one part so that we maximize the total number of their neighbors in the other part. We apply a GAT-based encoder to obtain node embeddings. This $\bx_\theta$ is the input to our \textsc{FW Decomposition} algorithm to obtain a small list of subsets $\{S^{(t)}\}_{t=0}^{T-1}$, of size $k$, and coefficients $\{\alpha_t\}_{t=0}^{T-1}$. For training, we minimize the differentiable surrogate loss $\mathcal{L}_{\text{COV}}=-\sum_{t=0}^{T-1} \alpha_t\,f_{\text{cov}}(\one_{S^{(t)}}) + \lambda \|\bx_\theta - \sum_{t=0}^{T-1} \alpha_t \one_{S^{(t)}}\|_2^2$  and backpropagate through the FW steps into the GAT parameters. Following \cite{karalias2025geometric} (\textbf{GeoNCO}), we train on synthetic random graphs and test on real-world graphs. For evaluation, \cref{tab:max-coverage-results}, we report $\max_t f_{\text{cov}}(\one_{S^{(t)}})$ without TTO and compare with the baselines without TTO. Our method is faster than \textbf{GeoNCO} and it often outperforms the learning baselines. Specifically, it consistently stays on the pareto front over the learning baselines and is also competitive against non-learning baselines considering the quality-efficiency tradeoff. (See Section~\ref{app:greedy-trap-ablation} for adversarial instances where greedy fails.)

\textbf{Ablation on projection for MC.} \cref{tab:ablation-mc-proj} shows that removing the polytope projection preserves solution quality while improving runtime. Although the model converges to a point outside the hypersimplex, the solution quality from our decomposition is slightly better than the projected version. This suggests the projection-free variant allows greater exploration by avoiding hard constraints.

\subsection{Quadratic Assignment Problem}
\label{sec:app-qap}

There are $n$ facilities and $n$ locations. Distances between locations are given by a matrix $\mathbf{B}\in\mathbb{R}^{n\times n}$, and flows between facilities by a matrix $\mathbf{A}\in\mathbb{R}^{n\times n}$. The goal is to assign each facility to a distinct location so as to minimize the total flow-weighted distance. Let $\Pi_n$ denote the set of $n\times n$ permutation matrices. The objective is $\min_{\mathbf{P} \in \Pi_{n}} f_{\text{qap}}(\mathbf{P})$ where $f_{\text{qap}}(\mathbf{P})=\mathrm{trace}(\mathbf{A} \mathbf{P} \mathbf{B} \mathbf{P}^\top)$ for all $\mathbf{P} \in \Pi_{n}$.


\textbf{Neural model and continuous targets.}
Given the output of encoder network, GraphSAGE, $\mathbf{X}_\theta\in [0,1]^{n\times n}$, we apply our \textsc{FW Decomposition} algorithm using a deterministic greedy algorithm for weighted maximum matching. After $T$ iterations, it returns a short list of permutation matrices $\{\mathbf{P}_{t}\}_{t=0}^{T-1}$ and coefficients $\{\alpha_t\}_{t=0}^{T-1}$. We train directly against the discrete QAP objective and using the following $\mathcal{L}_{\text{QAP}}(\mathbf{X}_\theta)
  =
  \sum_{t=0}^{T-1} \alpha_t \, f_{\text{qap}}(\mathbf{P}_t)+ \lambda \|\mathbf{X}_\theta - \sum_{t=0}^{T-1} \alpha_t \, \mathbf{P}_t\|_F^2$.
Gradients are backpropagated through the FW steps into the parameters that produce $\mathbf{X}_\theta$.
At evaluation we report $\min_t f_{\text{qap}}(\mathbf{P}_t)$. 

We train on synthetic instances following \cite{tan2024learning}. For size $n$, we sample locations $\mathbf{L} \sim \mathrm{Uniform}(0,1)^2$ and a symmetric flow matrix $\mathbf{A} \in \mathbb{R}^{n \times n}$ with zero diagonal, upper-triangular entries sampled from $[0,1]$, and entries independently set to zero with probability $p$ (symmetrically). Testing uses the standard \textbf{QAPLIB} benchmark \cite{burkard1997qaplib} (134 instances, 15 classes; see \cref{sec:qap-setup}).


Consistent with prior studies, we report per-class and aggregated gap statistics on instances of size 12–64, where baseline methods are tractable. Recent learning-based QAP methods, such as \textbf{RGM} \cite{liu2023revocable} and \textbf{SAWT} \cite{tan2024learning}, rely on iterative Learn-to-Construct or Learn-to-Improve pipelines which could be time consuming and prevent scalability. Our approach, on the other hand, produces a permutation in a single forward pass, which compared to the learning baselines, yields better efficiency, while achieving a better out of distribution generalization (OOD) quality. 

\begin{table*}[t]
\centering
\setlength{\tabcolsep}{3pt} 
\caption{Mean gaps and inference time (s) for models pre-trained on QAP32 and applied to QAPLIB instances. \textbf{Std.} is the standard deviation of the class-wise mean gaps across the 14 data classes (bur--wil). In each column, 
\textcolor{Red!50}{$\blacksquare$} indicates ranking the 1st, \textcolor{Blue!50}{$\blacksquare$} the 2nd.}
\label{tab:qap-mean-time}
\resizebox{\textwidth}{!}{
\begin{tabular}{l|cccccccccccccc|c|c|c}
\toprule
\textbf{Method} & \textbf{bur} & \textbf{chr} & \textbf{esc} & \textbf{had} & \textbf{kra} & \textbf{lipa} & \textbf{nug} & \textbf{rou} & \textbf{scr} & \textbf{sko} & \textbf{ste} & \textbf{tai} & \textbf{tho} & \textbf{wil} & \textbf{Avg.} & \textbf{Std.} & \textbf{Time (s)} \\
\midrule
SM    & 22.3 & 460.1 & 301.6 & 17.4 & 65.3 & 19.0 & 45.5 & 35.8 &123.4 & 29.0& 475.5 & 180.5 & 55.0 & 13.8 & 181.2 & 163.3 & \rkone{0.01} \\
RRWM  & 23.1 & 616.0 & 63.9  & 25.1 & 58.8 & 20.9 & 67.8 & 51.2 & 173.5 & 48.5 & 539.4 & 197.2 & 80.6 & 18.2 &169.5 & 192.9 & \rktwo{0.15} \\
SK-JA & 4.7  & \rkone{38.5} & 364.8 & 25.8 & 41.4 & \rkone{0.0}  & 25.3 & 13.7 & 48.6 & 18.3 & 120.4 & 25.2 & 32.9 & 8.8 & 93.2  & 93.9  & 563.4 \\
NGM   & 3.4  & 121.3 & 126.7 & 8.2  & 31.6 & 16.2 & 21.0 & 30.9 & 55.5 & 25.2 & 101.7 & 61.4 & 27.5 & 10.8 & 62.4  & 41.9  & 15.72 \\
RGM   & 7.1  & 112.4 & 32.8  & 6.2  & \rkone{15.0} & 13.3 & \rktwo{9.7}  & 13.4 & 45.5 & \rkone{10.6} & 134.1 & \rktwo{17.3} & \rkone{20.7} & \rktwo{8.1} & 35.8  & 40.4  & 75.53 \\
SAWT & \rktwo{2.8} & 110.7 & \rkone{13.5} & \rktwo{3.8} & \rktwo{30.1} & \rktwo{0.4} & \rkone{9.2} & \rkone{10.8} & \rktwo{28.5} & 17.7& 93.5 & \rkone{16.5} & 24.8 & \rkone{8.1} & \rktwo{26.8} & \rktwo{33.5} & 12.11 \\
\midrule
\textbf{FWNCO (ours)} & \rkone{2.3} & \rktwo{97.7} & \rktwo{19.0} & \rkone{2.8} & 37.0 & 13.5 & 11.5 & \rktwo{10.9} & \rkone{23.6} & \rktwo{15.0} & \rkone{83.5} & 20.7 & \rktwo{23.0} & 9.2 & \rkone{26.4} & \rkone{28.8} & 1.56 \\
\bottomrule
\end{tabular}
}
\end{table*}


In our experiments, \cref{tab:qap-mean-time}, our method achieves the best overall performance among all baselines, including both learning-based and classical solvers. The consistently low standard deviation of the average optimality gap across instance classes indicates stable behavior and good generalization across distributions. Moreover, our approach is the fastest neural method, running noticeably faster than other neural baselines. The only two methods with lower runtime, \textbf{SM} and \textbf{RRWM}, are purely traditional heuristics, which come at a substantial loss in solution quality. While \cref{tab:qap-mean-time} reports sizes 12--64, our method scales to larger instances ($>64$) with similar average gaps (\cref{tab:qaplib_gt64}). Most prior methods do not scale to such sizes with reasonable compute resources.

\textbf{Ablation on projection for QAP.} \cref{alg:fw} does not theoretically require its input to be in the Birkhoff polytope. However, we observed a weaker OOD generalization results for QAP due to the strong heterogeneity and distribution shift present in \textbf{QAPLIB}. Unconstrained outputs, i.e., "far from the polytope", cause the decomposition algorithm to return low quality permutations under drastic distributional shift. We therefore include a lightweight Sinkhorn projection layer that enforces approximate doubly-stochasticity, improving generalization. This step is introduced for practical considerations, not theoretical necessity.

\subsection{TSP via Spanning Trees and Learned Matchings}
\label{sec:app-tsp}
We consider spanning tree polytopes, where $\mc{C}$ is the set of spanning trees of a graph. For each $T \in \mc{C}$, the objective $f_{\text{TSP}}(T)$ is the cost of the TSP tour obtained deterministically via Christofides's algorithm \cite{christofides1976worst}. Training learns a distribution over spanning trees optimized for this objective.

\textbf{Algorithm-guided neural optimization for TSP.} One of the appeals of our approach is that it allows us to infuse learnable NN modules into an algorithmic framework. We adopt an algorithmically aligned, data-driven approach that integrates learning into the classical Christofides–Serdyukov algorithmic framework. Christofides's algorithm builds a TSP solution from a minimum spanning tree, whereas recent advances (e.g., \cite{asadpour2017log,karlin2021slightly,karlin2023deterministic}) sample from spanning tree distributions constructed from a subtour LP solution before applying variants of the Christofides pipeline. Our approach can be viewed as a learning-based analogue of this paradigm: instead of constructing such distributions analytically, we learn a distribution over spanning trees in an end-to-end manner, optimized directly for the final TSP objective. Moreover, we extend this idea further by incorporating a learning module for the matching step, effectively learning over matching polytopes as well. Given the neural network output, \textsc{FW Decomposition} returns a sparse distribution over spanning trees using Kruskal's algorithm as LMO.  We sample trees from this distribution, and recover tours by applying minimum-weight perfect matching on odd degree nodes followed by shortcutting. This pipeline defines a differentiable surrogate objective that mirrors the Christofides algorithm end to end. This also allows the perfect-matching step itself to be learned jointly, improving solution quality.

In our experiments, we consider Euclidean TSP instances $(G=(V,E), D)$, where $V$ is a set of points drawn uniformly at random from the unit square $[0,1]^2$, $E$ is the complete edge set on $V$, and $D$ is the corresponding symmetric distance matrix induced by Euclidean distances, which satisfies the triangle inequality. A tour is defined as a Hamiltonian cycle that visits each node exactly once, and its cost is given by the sum of the lengths of its constituent edges under $D$. The objective is to find a tour of minimum total cost.

Neural TSP methods commonly fall into two paradigms \cite{Ma2025COExpanderAS}: Learning-to-Construct (LC), which builds TSP tours sequentially with hard feasibility checks but suffers from high inference latency on large instances due to its inherently sequential decoding, and Global Prediction (GP), which predicts a full structure (e.g., a heatmap) and relies on post-processing, yielding fast inference but strong sensitivity to recovery heuristics. COExpander \cite{Ma2025COExpanderAS} combines LC and GP via Adaptive Expansion (AE), but still depends on heatmap recovery, which still is a combinatorial optimization subroutine, inheriting GP’s core limitation. Motivated by critiques of heatmap-guided post-hoc search \cite{xia2024position}, our method directly optimizes feasible tours end-to-end by differentiating through the selected FW branch, reducing reliance on hand-crafted post-processing.

We compare against four baseline categories: (i) exact solvers (\textbf{Concorde}, \textbf{Gurobi}); (ii) classical heuristics (\textbf{LKH3} \cite{helsgaun2017extension}, 500 trials); (iii) neural LC methods (\textbf{Sym-NCO} \cite{kim2022sym}), implemented in the RL4CO library \cite{berto2025rl4co}; and (iv) neural GP/AE and heatmap-based methods across RL and supervised settings, including \textbf{DIMES} \cite{qiu2022dimes}, \textbf{DIFUSCO} \cite{sun2023difusco}, and \textbf{COExpander}, a SOTA AE method. In \cref{sec:inference_settings}, we provide discussion regarding approaches that use inference optimization techniques such as Gradient-based Active Search, and Search-based Refinement such as MCTS and 2-opt. 

\begin{table*}[t]
\centering
\setlength{\tabcolsep}{3pt}
\caption{Performance comparison on TSP benchmarks across different scales. Results include optimality gap (\%) and inference time. Methods are categorized as Unsupervised Learning (UL), Exact, Heuristics, Reinforcement Learning (RL), and Supervised Learning (SL). Among neural methods, \textcolor{Red!50}{$\blacksquare$} indicates ranking the 1st, \textcolor{Blue!50}{$\blacksquare$} the 2nd in each column.}
\label{tab:tsp-results}
\resizebox{1\textwidth}{!}{
\begin{tabular}{ll ccc ccc ccc}
\toprule
& & \multicolumn{2}{c}{\textbf{TSP-50}} & \multicolumn{2}{c}{\textbf{TSP-100}} & \multicolumn{2}{c}{\textbf{TSP-500}} & \multicolumn{2}{c}{\textbf{TSP-1000}} \\
\cmidrule(r){3-4} \cmidrule(r){5-6} \cmidrule(r){7-8} \cmidrule(r){9-10}
\textbf{Method} & \textbf{Type} & Gap & Time & Gap & Time & Gap & Time& Gap & Time \\
\midrule
Concorde (Optimal) & Exact & $opt$ & 0.04s & $opt$ & 0.21s & $opt$ & 17.91s & $opt$ & 435.61s \\
Gurobi (Optimal) & Exact & $opt$ & 0.27s & $opt$ & 1.14s & N/A & N/A & N/A & N/A  \\
\midrule
LKH3 (500) & Heuristics & 0.01\% & 0.03s & 0.01\% & 0.05s & 0.93\% & 0.32s & 1.30\% & 1.06s \\
\midrule
DIFUSCO & SL & \rktwo{0.53\%} & 0.43s & \rktwo{1.24\%} & 0.66s & 9.79\% & 2.19s & 10.90\% & 7.68s \\
COExpander ($S=1,D_s=3,I_s=5$) & SL & \rkone{0.02\%} & \rktwo{0.11s} & \rkone{0.07\%} & \rktwo{0.20s} & \rkone{3.80\%} & \rktwo{0.69s} & \rkone{6.53\%} & \rktwo{2.50s} \\
DIMES & RL & 13.71\% & \rkone{0.05s} & 11.60\% & \rkone{0.08s} & 15.34\% & \rkone{0.39s} & 13.97\% & \rkone{0.79s} \\
RL4CO (Sym-NCO) & RL & 1.23\% & 0.36s & 13.78\% & 0.66s & N/A & N/A &  N/A & N/A \\

\midrule
\textbf{FWNCO (ours)} & UL & 1.81\% & 0.23s & 3.48\% & 0.26s & 7.30\%& 1.27s & 8.27\% & 4.26s \\
\textbf{FWNCO+learned matching (Ours)} & UL & 1.40\% & 0.52s & 2.84\% & 0.74s & \rktwo{6.90\%} & 2.77s & \rktwo{7.43\%} & 9.16s \\
\bottomrule
\end{tabular}
}

\end{table*}

Our results, \cref{tab:tsp-results}, show that our method consistently outperforms unsupervised neural baselines including RL. On larger instances, it also surpasses \textbf{DIFUSCO}, a strong diffusion-based approach trained with supervision. While we are still behind the strongest supervised baselines, e.g. \textbf{COExpander}, our approach is more resource-efficient. In particular, \textbf{DIFUSCO} and \textbf{COExpander} rely on up to $\sim$1.5M labeled instances (depending on the instance size) and require extensive training time on good hardware, whereas our model can be trained in under one hour on a single outdated GPU.

\textbf{Ablation for TSP}.
We examine the effect of projecting the neural network output onto the spanning tree polytope and find that removing projection slightly improves performance on both TSP-500 and TSP-1000 (\cref{tab:ablation_proj}). We also study generalization across instance sizes and observe strong cross-size generalization (\cref{tab:generalization}). Our model uses subtour LP solutions \cite{dantzig1954solution} as features, obtained via the QSopt LP solver, also used by Concorde. Despite this shared solver, our learning pipeline is significantly faster on large instances. Our ablation study, \cref{sec:lp+noise}, confirm that these gains arise from learned structures rather than arbitrary perturbations of the LP relaxation. 

\section{Conclusion}
\label{sec:conclusion}

We propose a general framework for self-supervised CO under constraints. Our approach is flexible in that it does not require the neural network output to lie in a feasible polytope. It uses a generic procedure, to decompose the output and obtain a sparse distribution over feasible solutions, effectively enabling SSL. Our method shows strong empirical performance compared to neural baselines on several CO problems. We present an example of infusing learned modules within an algorithm in an end-to-end pipeline, showcasing this approach on TSP using the Christofides algorithm. We believe this paves the way toward new NN-infused algorithms that can effectively tackle complex CO problems. We note that model architecture can affect performance; we do not study this here and leave it for future work.


\bibliographystyle{plainnat}
\bibliography{references}
\appendix

\section{Proofs for \cref{thm:FWextension} and Corollary \ref{cor:approx-fw-extension}}
\label{sec:proof}

\begin{proof}[Proof of \cref{thm:FWextension}]
The convex-combination structure follows directly from the update
$\bu_t=(1-\gamma_t)\bu_{t-1}+\gamma_t\bv_t$ with $\gamma_t\in[0,1]$ and the
initialization $\bu_0=\bv_0\in\ext(\mathcal P)$. Inductively,
\[
\bu_{T-1}=\sum_{t=0}^{T-1}\alpha_t\bv_t,\qquad
\bv_t\in\ext(\mathcal P),\qquad
\alpha_t\ge 0,\qquad
\sum_{t=0}^{T-1}\alpha_t=1.
\]

For the convergence rate, let $D=\diam(\mathcal P)$ and
$J(\bu)=\frac12\|\bx-\bu\|_2^2$. The gradient of $J$ is $1$-Lipschitz, so the
Frank--Wolfe curvature of $J$ over $\mathcal P$ is bounded by $D^2$. The standard
Frank--Wolfe rate with exact linear minimization/maximization oracles and line
search gives, after $k$ updates,
\[
J(\bu_k)-J(\bu^\star)\le \frac{2D^2}{k+2},
\qquad
\bu^\star\in\argmin_{\bu\in\mathcal P}J(\bu).
\]
Setting $k=T-1$ yields
\[
J(\bu_{T-1})-J(\bu^\star)\le \frac{2D^2}{T+1}.
\]

It remains to justify the almost-everywhere differentiability claim under the
stable oracle decision assumption. For almost every $\bx$, the assumption gives
a neighborhood on which all selected vertices and line-search clipping regimes
are fixed. Fix such an input and neighborhood. The initialization $\bu_0=\bv_0$
is constant on this neighborhood. Recursively, if $\bu_{t-1}(\bx)$ is smooth
there, then
\[
\br_t(\bx)=\bx-\bu_{t-1}(\bx),\qquad
\mathbf{d}_t(\bx)=\bv_t-\bu_{t-1}(\bx)
\]
are smooth because $\bv_t$ is fixed. When $\mathbf{d}_t(\bx)\neq 0$, the unclipped line-search
value
\[
\tilde\gamma_t(\bx)
=
\frac{\langle \br_t(\bx),\mathbf{d}_t(\bx)\rangle}{\|\mathbf{d}_t(\bx)\|_2^2}
\]
is a smooth algebraic function on a possibly smaller neighborhood with nonzero
denominator. When $\mathbf{d}_t=0$, we use the algorithmic convention
$\gamma_t=0$, and the update is locally constant in that direction. Since the
clipping regime is fixed on the neighborhood, $\gamma_t$ is either $0$, $1$, or
$\tilde\gamma_t$, and is therefore differentiable there. Hence
\[
\bu_t=(1-\gamma_t)\bu_{t-1}+\gamma_t\bv_t
\]
involves only differentiable algebraic operations. The same recursive update
shows that each coefficient $\alpha_t(\bx)$ is differentiable on the
neighborhood. The only excluded inputs lie in the zero-volume switching set from
the stable oracle decision assumption, so the coefficient map
$\bx\mapsto\boldsymbol{\alpha}(\bx)$ is differentiable almost everywhere.
\end{proof}

\begin{proof}[Proof of Corollary \ref{cor:approx-fw-extension}]
The convex-combination conclusion is unchanged because the update still uses
$\gamma_t\in[0,1]$ and vertices of $\mathcal P$. The relative FW-gap condition
\[
\langle \br_t,\bv_t-\bu_{t-1}\rangle
\ge
\delta\max_{\bv\in\mathcal P}\langle \br_t,\bv-\bu_{t-1}\rangle
\]
is the maximization form of the standard approximate Frank--Wolfe oracle
condition. Applying the approximate Frank--Wolfe theorem of
\citet{locatello2017unified} to
$J(\bu)=\frac12\|\bx-\bu\|_2^2$, with Lipschitz constant $1$ and diameter $D$,
gives after $T-1$ updates
\[
J(\bu_{T-1})-J(\bu^\star)
\le
\frac{2\left(D^2/\delta+h_0\right)}{\delta(T-1)+2},
\qquad
h_0=J(\bu_0)-J(\bu^\star).
\]
The almost-everywhere differentiability statement follows from the same stable
oracle decision argument used above.
\end{proof}

\FloatBarrier
\section{Details and Further Experiments for Maximum Coverage }
\paragraph{Hardware.}
All experiments, including training and inference on Maximum Coverage, TSP and QAP, are done on 16 cores (32 threads) of Intel(R) Xeon(R) Platinum 8268 CPU (24 cores, 48 threads in total), 32 GB DDR4 ram, with a single Nvidia RTX8000 48GB GPU.

\subsection{Experiment Setup}
\paragraph{Datasets.}
Following prior work~\citep{bu2024tackling,wang2022towards,karalias2025geometric}, we evaluate on synthetic and real-world bipartite graphs $(U,V,E)$, where $V$ is the ground set and the task is to select $k$ nodes from $V$ to maximize the total weight of covered nodes in $U$. For synthetic data, we use the Random Uniform dataset under two scales, \texttt{Random500} and \texttt{Random1000}. Each setting contains a training dataset of $100$ independently generated instances with $|V|=500, |U|=1000$ (\texttt{Random500}) or $|V|=1000, |U|=2000$ (\texttt{Random1000}). Each $u\in U$ is assigned an integer weight sampled uniformly from $\{1,\ldots,100\}$, and each $v\in V$ covers a random subset of $U$ whose cardinality is sampled uniformly from $\{10,\ldots,30\}$. Testing are conducted on datasets sampled from the same distribution. For real-world testing, we additionally evaluate on three \emph{Railway} instances derived from Italian railway crew assignment graphs: \texttt{rail507} ($|V|=507, |U|=63009$), \texttt{rail516} ($|V|=516, |U|=47311$), and \texttt{rail582} ($|V|=582, |U|=55515$).

\paragraph{Training.} In general, we use the same training setup as \cite{karalias2025geometric}, using a 3-layer GNN with residual connection. To enhance optimization stability and solution quality, we incorporate specific regularization strategies to promote exploration within the hypersimplex. An entropy regularization term is applied with a cosine-decayed weight $\lambda_t$, which is initialized at $0.05$ and annealed to zero by the 30th epoch. A dropout rate of $0.1$ is also applied exclusively during the training phase.

The model is trained for a total of $80$ epochs using the AdamW optimizer with a weight decay of $1\times 10^{-4}$. We employ a batch size of $4$ and a fixed random seed of $42$ to ensure reproducibility. The learning rate follows a \texttt{warmup\_cosine} scheduler, which includes $50$ warmup epochs and decays from a peak of $5\times 10^{-3}$ to a minimum of $5\times 10^{-5}$. Data loading and preprocessing are handled efficiently by $16$ CPU worker threads.

\paragraph{Inference.}
In inference, we run our \textsc{FW Decomposition} for 50 iterations for each instance, and select the best result out of the decomposed sets. We do not use any randomized or local improvment method.
\subsection{Baselines}
We compare against three classes of baselines that span exact solvers, classical heuristics, and neural approaches.

\textbf{(i) Exact solvers (non-learning).}
We employ \textbf{Gurobi} to obtain exact MIP-based solutions, imposing a time limit of 120 seconds for each graph.

\textbf{(ii) Classical heuristics (non-learning).}
We include a \textbf{Random} baseline, which samples subsets of size $k$ uniformly at random over multiple trials and selects the best solution found within 240 seconds.
We also report the \textbf{Greedy algorithm} \citep{NemhauserWF78}, which iteratively adds the element with the largest marginal gain to achieve the well-known $(1-\tfrac{1}{\mathrm{e}})$-approximation guarantee.
Additionally, we classify \textbf{UCOM2} \citep{bu2024tackling} as a non-learning baseline due to its use of incremental greedy de-randomization, and we report results for its short variant.

\textbf{(iii) Neural baselines.}
We evaluate a diverse set of neural methods.
We include \textbf{CardNN} \citep{wang2022towards} variants (without Test-Time Optimization, following \cite{bu2024tackling}) and a naive version of \textbf{EGN} \citep{karalias2020erdos} with naive probabilistic objective construction and iterative rounding.
We also compare against \textbf{GeoNCO} \citep{karalias2025geometric}, which employs the same Carath\'eodory decomposition framework as our method but lacks our proposed enhancements.
Finally, we include a Reinforcement Learning (\textbf{RL}) baseline following \citep{karalias2025geometric} trained with the Actor-Critic \citep{actor_critic} algorithm on the same problem instances.
\subsection{Ablation study on projection}
We ablate the effect of explicitly projecting the network output back into the feasible solution polytope before running \textsc{FW Decomposition}.We use the same projection method as \cite{karalias2025geometric}, while they rely on such a projection
to enforce feasibility, our framework does not require it since feasibility is already guaranteed by the decomposition.

Table~\ref{tab:ablation-mc-proj} shows that removing projection preserves solution quality and in fact yields
slightly better coverage across both \texttt{rand500} and \texttt{rand1000} settings, with slightly lower
wall-clock inference time. We attribute this to the fact that the projection is a hard, piecewise operator that
can compress the effective search region and reduce gradient signal diversity; in our pipeline, it becomes largely
redundant because feasibility is enforced by the decomposition itself. By contrast, allowing unconstrained continuous
scores retains a richer exploration space for the model, while \textsc{FW Decomposition} still produces feasible
candidates, leading to steadier optimization and marginally improved performance in practice.

\begin{table}[H]
    \centering
    \footnotesize
    \setlength{\tabcolsep}{6pt}
    \caption{Ablation study results. Grouped by dataset size.}
    \label{tab:ablation-mc-proj}
    \begin{tabular}{lcccc}
        \toprule
        \multirow{2}{*}{Method} & \multicolumn{2}{c}{rand500} & \multicolumn{2}{c}{rand1000} \\
        \cmidrule(lr){2-3} \cmidrule(lr){4-5}
                                & $k=10$   & $k=50$   & $k=20$   & $k=100$ \\
        \midrule
        \multirow{2}{*}{With Projection}   & 15190.10  & 42138.45 & 29906.24 & 83474.98 \\
                                        & 0.4635s   & 0.5513s  & 0.6423s & 0.9940s  \\
        \cmidrule(lr){1-5}
        \multirow{2}{*}{No Projection} & 15192.64  & 42172.44 & 29926.33 & 83559.22 \\
                                        & 0.4261s   & 0.5378s   & 0.6472s & 0.9135s  \\
        \bottomrule
    \end{tabular}
\end{table}

\subsection{Ablation on Hard Instances for Greedy}
\label{app:greedy-trap-ablation}

\paragraph{Hard example for greedy.}
The greedy algorithm performs well, especially on random graphs, for the MC problem because the coverage function is monotone and submodular. To show our method works even when greedy fails, we construct the \textsc{GreedyTrap500} and \textsc{GreedyTrap1000} datasets for MC. In this construction, the greedy algorithm is pushed near its theoretical $1-1/e=0.63212$ lower bound and attracted by sets with large initial marginal coverage, falling significantly short of the optimal solution.

The construction is as follows. We first choose a coverable subset of elements from $\mathcal{U}$ and partition it into $k$ disjoint parts. For each part, we add one set that covers the entire part; selecting these $k$ sets covers all coverable elements. We then add $k$ greedy-trap sets. Each trap set covers slightly more than a $1/k$ fraction of every part, so these sets have large marginal gain at the beginning and are selected by greedy. The remaining sets are low-degree filler sets. For \textsc{GreedyTrap500}, we use 500 sets and 1000 elements; for \textsc{GreedyTrap1000}, we use 1000 sets and 2000 elements. In testing, we report the coverage ratio over all elements.

\begin{table}[htb]
\centering
\caption{Performance comparison on the \textsc{GreedyTrap} datasets. Values
represent coverage ratios in percentage averaged across 64 test instances.}
\label{tab:greedytrap}
\begin{tabular}{lcccc}
\toprule
 & \multicolumn{2}{c}{\textbf{GreedyTrap500}} &
   \multicolumn{2}{c}{\textbf{GreedyTrap1000}} \\
\cmidrule(lr){2-3} \cmidrule(lr){4-5}
\textbf{Method} & $k=10$ & $k=50$ & $k=20$ & $k=100$ \\
\midrule
Greedy & 46.30\% & 66.50\% & 45.80\% & 70.60\% \\
FWNCO(ours)   & 59.90\% & 78.40\% & 55.60\% & 82.30\% \\
OPT    & 60.10\% & 81.90\% & 58.60\% & 86.70\% \\
\bottomrule
\end{tabular}
\end{table}

As shown in Table~\ref{tab:greedytrap}, greedy falls significantly short of
OPT on these hard instances, whereas our method consistently
outperforms greedy and achieves coverage much closer to OPT. This shows that
our method can still learn useful problem structure and maintain strong
performance on data distributions where the standard greedy heuristic fails.
\subsection{Effect of the Frank--Wolfe Iteration Budget}
\label{app:mc-fw-budget}

We investigate the effect of the number of Frank--Wolfe
decomposition iterations $T$ on solution quality and runtime
for Maximum Coverage. Each iteration of Algorithm~\ref{alg:fw} performs
one LMO call and a closed-form line-search step. Note that the runtime is linear in $T$.

Table~\ref{tab:mc-fw-iterations} reports results for four
configurations, with solution quality normalized by the
weighted coverage obtained at $T=200$. Increasing $T$
consistently improves quality, with diminishing gains at
larger budgets. At $T=50$, the method already achieves
99.391--99.652\% of the reference quality across all four
configurations. Increasing the budget to $T=200$ improves
relative quality by only 0.348--0.609 percentage points,
while increasing the reported runtime by approximately
$3.7$--$3.8\times$. These results support our choice of
$T=50$ as a practical balance between solution quality
and computational cost.

\begin{table*}[h]
\centering
\caption{Maximum coverage quality and inference time under different FW
iteration budgets. Quality is normalized by the weighted coverage obtained at
$T=200$ and reported as a percentage. Time recorded on an NVIDIA GeForce RTX 4090 GPU and not directly comparable with other results.}
\label{tab:mc-fw-iterations}
\setlength{\tabcolsep}{3.5pt}
\resizebox{\textwidth}{!}{%
\begin{tabular}{rcccccccc}
\toprule
& \multicolumn{2}{c}{$n{=}500,k{=}10$}
& \multicolumn{2}{c}{$n{=}500,k{=}50$}
& \multicolumn{2}{c}{$n{=}1000,k{=}20$}
& \multicolumn{2}{c}{$n{=}1000,k{=}100$} \\
\cmidrule(lr){2-3}\cmidrule(lr){4-5}
\cmidrule(lr){6-7}\cmidrule(lr){8-9}
$T$
& Quality (\%) & Time (s)
& Quality (\%) & Time (s)
& Quality (\%) & Time (s)
& Quality (\%) & Time (s) \\
\midrule
10  & 98.197 & 0.0802
    & 99.230 & 0.1153
    & 98.712 & 0.1418
    & 99.464 & 0.1965 \\
30  & 98.942 & 0.2046
    & 99.376 & 0.2853
    & 99.296 & 0.3556
    & 99.548 & 0.4920 \\
\textbf{50}
    & \textbf{99.391} & \textbf{0.3295}
    & \textbf{99.494} & \textbf{0.4554}
    & \textbf{99.501} & \textbf{0.5669}
    & \textbf{99.652} & \textbf{0.7852} \\
70  & 99.470 & 0.4516
    & 99.630 & 0.6246
    & 99.601 & 0.7784
    & 99.755 & 1.0770 \\
100 & 99.659 & 0.6333
    & 99.771 & 0.8775
    & 99.695 & 1.0889
    & 99.861 & 1.5150 \\
150 & 99.859 & 0.9345
    & 99.909 & 1.2994
    & 99.848 & 1.5979
    & 99.930 & 2.2447 \\
\textbf{200}
    & \textbf{100.000} & \textbf{1.2365}
    & \textbf{100.000} & \textbf{1.7197}
    & \textbf{100.000} & \textbf{2.1068}
    & \textbf{100.000} & \textbf{2.9653} \\
\bottomrule
\end{tabular}%
}
\end{table*}

\FloatBarrier
\section{Details and Further Experiments on QAP}

In this section, we provide a detailed discussion of the numerical results for QAP. The full experimental results are reported in \cref{QAPLIB}.


\begin{table*}[t]
\centering
\caption{Complete experiment results on QAPLIB. The mean/max/min
gaps are reported for each category. The mean performance over all categories and the inference time (s)
per instance are reported.}
\resizebox{1\textwidth}{!}{
\begin{tabular}{l|ccc|ccc|ccc|ccc}
\hline & \multicolumn{3}{|c|}{ bur (26) } & \multicolumn{3}{|c|}{chr(12-25)} & \multicolumn{3}{|c|}{ esc (16-64)} & \multicolumn{3}{|c}{ had (12 - 20) }  \\
\cline{2-13} & Mean$\downarrow$ & Min$\downarrow$  &  Max$\downarrow$  & Mean$\downarrow$ & Min$\downarrow$  &  Max$\downarrow$ & Mean$\downarrow$ & Min$\downarrow$  &  Max$\downarrow$  & Mean$\downarrow$ & Min$\downarrow$  &  Max$\downarrow$ \\
\hline SM & 22.3 & 20.3 & 24.9 & 460.1 & 144.6 & 869.1 & 301.6 & \textbf{0.0} & 3300.0 & 17.4 & 14.7 & 21.5 \\
 RRWM & 23.1 & 19.3 & 27.3 & 616.0 & 120.5 & 1346.3 & 63.9 & \textbf{0.0} & 200.0 & 25.1 & 22.1 & 28.3 \\
 SK-JA & 4.7 & 2.8 & 6.2 & \textbf{38.5} & \textbf{0.0} &\textbf{186.1} & 364.8 & \textbf{0.0} & 2200.0 & 25.8 & 6.9 & 100.0  \\
 NGM & 3.4 & 2.8 & 4.4 & 121.3 & 45.4 & 251.9 & 126.7 & \textbf{0.0} & 200.0 & 8.2 & 6.0 & 11.6  \\
  RGM & 7.1 & 4.5 & 9.0 & 112.4 & 23.4 & 361.4 & 32.8 & \textbf{0.0} & 141.5 & 6.2 & 1.9 & 9.0 \\
 SAWT & 2.8 & 2.2 & 3.4 & 110.7 & 8.6 & 201.2 & \textbf{13.5} & \textbf{0.0} & \textbf{63.7} & 3.8 & 1.6 & 6.5 \\
 \hline
 \textbf{FWNCO (ours)} & \textbf{2.3} & \textbf{1.9} & \textbf{2.6} & 97.7 & 28.8 & 209.9 & 19.1 & \textbf{0.0} & 119.1 & \textbf{2.8} & \textbf{1.2} & \textbf{3.7} \\
\hline
\hline
&\multicolumn{3}{|c|}{ kra (30-32)} & \multicolumn{3}{|c|}{ lipa (20-60)} & \multicolumn{3}{|c|}{ nug (12-30)} & \multicolumn{3}{|c}{ rou (12-30)} \\
\cline{2-13}
& Mean$\downarrow$ & Min$\downarrow$  &  Max$\downarrow$  & Mean$\downarrow$ & Min$\downarrow$  &  Max$\downarrow$ & Mean$\downarrow$ & Min$\downarrow$  &  Max$\downarrow$  & Mean$\downarrow$ & Min$\downarrow$  &  Max$\downarrow$\\
\hline SM & 65.3 & 63.8 & 67.3 & 19.0 & 3.8 & 34.8 & 45.5 & 34.2 & 64.0 & 35.8 & 30.9 & 38.2 \\
RRWM  &58.8 & 53.9 & 67.7 & 20.9 & 3.6 & 41.2 & 67.8 & 52.6 & 79.6 & 51.2 & 39.3 & 60.1 \\
SK-JA & 41.4 & 38.9 & 44.4 & \textbf{0.0} & \textbf{0.0} & \textbf{0.0} & 25.3 & 10.9 & 100.0 & 13.7 & 10.3 & 17.4 \\
 NGM & 31.6 & 28.7 & 36.8 & 16.2 & 3.6 & 29.4 & 21.0 & 14.0 & 28.5 & 30.9 & 23.7 & 36.3 \\
  RGM & \textbf{15.0} & \textbf{10.4} & \textbf{20.6} & 13.3 & 3.0 & 23.8 & 9.7 & 6.1 & \textbf{12.9} & 13.4 & \textbf{7.1} & 16.7 \\
 SAWT & 30.1 & 28.1 & 34.2 & 0.4 & \textbf{0.0} & 2.6 & \textbf{9.2} & \textbf{4.1} & 13.1 & \textbf{10.8} & 8.2 & 13.1 \\
  \hline
 \textbf{FWNCO (ours)} & 37.0 & 34.8 & 39.2 & 13.5 & 2.2 & 26.3 & 11.5 & 7.3 & 14.7 & 10.9 & 9.7 & \textbf{12.1} \\
\hline
\hline
&\multicolumn{3}{|c|}{ scr (12-20)} & \multicolumn{3}{|c|}{ sko(42-64)} & \multicolumn{3}{|c|}{ ste(36)} & \multicolumn{3}{|c}{ tai(12-64)} \\
\cline{2-13}
& Mean$\downarrow$ & Min$\downarrow$  &  Max$\downarrow$  & Mean$\downarrow$ & Min$\downarrow$  &  Max$\downarrow$ & Mean$\downarrow$ & Min$\downarrow$  &  Max$\downarrow$  & Mean$\downarrow$ & Min$\downarrow$  &  Max$\downarrow$\\
\hline SM & 123.4 & 104.0 & 139.1 & 29.0 & 26.6 & 31.4 & 475.5 & 197.7 & 1013.6 & 180.5 & 21.6 & 1257.9 \\
RRWM  & 173.5 & 98.9 & 218.6 & 48.5 & 47.7 & 49.3 & 539.4 & 249.5 & 1117.8 & 197.2 & 26.8 & 1256.7 \\
SK-JA & 48.6 & 44.3 & 55.7 & 18.3 & 16.1 & 20.5 & 120.4 & 72.5 & 200.4 & 25.2 & 1.6 & 107.1\\
 NGM & 55.5 & 41.4 & 66.2 & 25.2 & 22.8 & 27.7 & 101.7 & 57.6 & 172.8 & 61.4 & 18.7 & 352.1 \\
  RGM & 45.5 & 30.2 & 56.1 & \textbf{10.6} & \textbf{9.9} & \textbf{11.2} & 134.1 & 69.9 & 237.0 & 17.3 & 11.4 & \textbf{28.6} \\
 SAWT & 28.5 & \textbf{10.2} & 48.0 & 17.7 & 16.7 & 18.7 & 93.5 & \textbf{46.7} & 170.2 & \textbf{16.5} & \textbf{0.5} & 48.7 \\
  \hline
 \textbf{FWNCO (ours)} & \textbf{23.6} & 13.6 & \textbf{39.8} & 15.0 & 13.8 & 15.6 & \textbf{83.5} & 58.5 & \textbf{130.2} & 20.7 & 1.7 & 43.0 \\
\hline
\hline
&\multicolumn{3}{|c|}{ tho (30-40)} & \multicolumn{3}{|c|}{ wil(50)} & \multicolumn{3}{|c|}{ \textbf{Average(12-64)}} & \multicolumn{3}{|c}{Time per instance} \\
\cline{2-10}
& Mean$\downarrow$ & Min$\downarrow$  &  Max$\downarrow$  & Mean$\downarrow$ & Min$\downarrow$  &  Max$\downarrow$ & Mean$\downarrow$ & Min$\downarrow$  &  Max$\downarrow$    &  \multicolumn{3}{|c}{(in seconds)} \\
\hline SM & 55.0 & 54.0 & 56.0 & 13.8 & 11.7 & 15.9 & 181.2 & 46.9 & 949.9 & & \textbf{0.01}& \\
RRWM & 80.6 & 78.2 & 83.0 & 18.2 & 12.5 & 23.8 & 169.5 & 49.5 & 432.9 & & 0.15& \\
SK-JA & 32.9 & 30.6 & 35.3 & 8.8 & \textbf{6.7} & 10.7 & 93.2 & \textbf{9.0} & 497.9 & & 563.4& \\
 NGM & 27.5 & 24.8 & 30.2 & 10.8 & 8.2 & 11.1 & 62.4 & 17.8 & 129.7 & & 15.72& \\
  RGM & \textbf{20.7} & \textbf{12.7} & 28.6 & 8.1 & 7.9 & \textbf{8.4} & 35.8 & 10.7 & 101.1 & & 75.53& \\
 SAWT & 24.8 & 23.2 & 26.4 & \textbf{8.1} & 7.6 & 8.6 & 26.8 & 11.2 & \textbf{47.0} && 12.11 & \\
  \hline
 \textbf{FWNCO (ours)} & 23.0 & 20.7 & \textbf{25.3} & 9.2 & 9.2 & 9.2 & \textbf{26.4} & 14.5 & 49.3 & & 1.56 & \\
\hline
\end{tabular}    
}
\label{QAPLIB}

\end{table*}

\subsection{Experiment Setup}
\label{sec:qap-setup}
\paragraph{Instance distribution and data splits.}
For synthetic training, we follow the data generation method of \cite{tan2024learning} and generate a synthetic training pool with 5120 instances, $n=32$ and $p=0.3$. During training, each gradient step samples a mini-batch uniformly from the training pool.

For test, we use the QAPLIB dataset to demostrate the generalization ability of our method. QAPLIB~\citep{burkard1997qaplib} is a widely used benchmark for QAP, collecting
real-world and synthetic instances from multiple sources.
Each instance is specified by a flow matrix and a distance matrix, and the naming convention uses a class prefix
(e.g., \texttt{bur}, \texttt{chr}, \texttt{tai}) together with the problem size.
The distributions and structural properties of the flow/distance matrices vary substantially across classes,
making QAPLIB a challenging testbed for cross-class generalization.

\paragraph{Input representation.}
Given $A$ and $B$, we build a node embedding for each matrix separately using
per-node statistics (diagonal entry and row/column aggregations) together with random-projection features computed with a fixed Gaussian
projection matrix $R$. Features are computed after per-instance magnitude normalization
(i.e., dividing by $\max_{i,j}|M_{ij}|$) to stabilize training across varying scales.

\paragraph{Model, optimizer, and training schedule.}
We use a two-layer GraphSAGE aggregator with a residual skip connection and hidden size
\texttt{model\_dim}=128, applied independently to $A$ and $B$.
The neural network outputs a matrix followed by a Sinkhorn projection layer.
We train with AdamW (learning rate $3\cdot 10^{-4}$, weight decay $10^{-4}$) and gradient
clipping at norm $1.0$. Training runs for 100 epochs; each epoch consists of 100 instances randomly sampled from the training pool. A single global RNG seed (0) is used for both training and inference.

\paragraph{Sinkhorn as a practical projection layer.}
Although our FW-based pipeline does \emph{not} theoretically require the network output to lie inside the Birkhoff polytope (the \textsc{FW Decomposition} can be applied to any real matrix), in QAP this relaxation is empirically fragile because QAP instances are substantially more heterogeneous than the other tasks: QAPLIB in particular exhibits highly dispersed scales, sparsity patterns, and instance-generation mechanisms across data classes, leading to large distribution gaps relative to synthetic pretraining.
In this regime, leaving the model output unconstrained often yields score matrices whose magnitude and row/column mass distribution are poorly calibrated, which in turn destabilizes the downstream greedy steps and makes the distribution returned by \textsc{FW Decomposition} less informative.

To mitigate this, we insert a Sinkhorn \cite{sinkhorn1964relationship} projection layer with 20 iterations as a soft projection that maps raw assignment logits to an approximately doubly-stochastic matrix. This operation enforces nonnegativity and approximately unit row/column sums, placing $\mathbf{X}$, the output of neural network, closer to the Birkhoff polytope. It serves as a lightweight output-range calibration mechanism that significantly improves numerical stability under the large inter-class distribution shift in QAPLIB.
We emphasize that this Sinkhorn step is used for practical robustness, not as a theoretical requirement of the \textsc{FW Decomposition} itself.

\paragraph{Decomposition settings.}
We run our \textsc{FW Decomposition} algorithm over permutation matrices with a iteration budget $T=2000$ in training and $T=3000$ in inference.

At each iteration, the LMO is instantiated by a \emph{greedy} maximum-weight assignment routine (no Hungarian), applied to the current residual matrix.
The resulting decomposition yields a sequence of permutations and their weights; in inference, we select the minimum-cost permutation under the QAP objective for reporting.

\begin{table}[ht!]
\centering
\caption{QAPLIB results for instances with $n>64$.}
\label{tab:qaplib_gt64}
\begin{tabular}{l r r r}
\toprule
Instance & pred\_best & best\_known & gap \\
\midrule
lipa70a  & 172973      & 169755      & 1.90\%   \\
lipa70b  & 5872840     & 4603200     & 27.58\%  \\
sko72    & 75952       & 66256       & 14.63\%  \\
lipa80a  & 257769      & 253195      & 1.81\%   \\
lipa80b  & 9983972     & 7763962     & 28.59\%  \\
tai80b   & 1057836224  & 818415043   & 29.25\%  \\
tai80a   & 15311664    & 13499184    & 13.43\%  \\
sko81    & 104214      & 90998       & 14.52\%  \\
lipa90a  & 366329      & 360630      & 1.58\%   \\
lipa90b  & 16017197    & 12490441    & 28.24\%  \\
sko90    & 132078      & 115534      & 14.32\%  \\
sko100a  & 174044      & 152002      & 14.50\%  \\
sko100b  & 174378      & 153890      & 13.31\%  \\
sko100c  & 170276      & 147862      & 15.16\%  \\
sko100d  & 170574      & 149576      & 14.04\%  \\
sko100e  & 170746      & 149150      & 14.48\%  \\
sko100f  & 169558      & 149036      & 13.77\%  \\
tai100a  & 23638054    & 21052466    & 12.28\%  \\
tai100b  & 1498196608  & 1185996137  & 26.32\%  \\
wil100   & 296136      & 273038      & 8.46\%   \\
esc128   & 244         & 64          & 281.25\% \\
tai150b  & 619556288   & 498896643   & 24.19\%  \\
tho150   & 9579958     & 8133398     & 17.79\%  \\
tai256c  & 50753096    & 44759294    & 13.39\%  \\
\midrule
\multicolumn{3}{r}{Average gap} & 26.87\% \\
\bottomrule
\end{tabular}
\end{table}

\paragraph{QAPLIB evaluation.}

Following previous approaches \cite{tan2024learning,liu2023revocable}, We report per-class gap statistics and the aggregate summary as in \cref{QAPLIB} for instances of size 12-64 since many baseline methods will run out of time on very large instances. However, our method is also capable of solving instances with larger sizes, as listed below in \cref{tab:qaplib_gt64}.

\subsection{Related Work (QAP / Matching).}
QAP and matching are among the most challenging CO problems for learning-based methods, largely due to severe
distribution shifts and outliers such as those in QAPLIB.
\textbf{NGM}~\cite{wang2021neural} was one of the earliest attempts to tackle QAP/graph matching with neural models. Subsequent work proposed \textbf{RGM}~\cite{liu2023revocable}, which introduced a \emph{learn-to-construct (L2C)} pipeline for QAP, building a permutation by making a sequence of constructive decisions, following that, \textbf{SAWT}~\cite{tan2024learning} introduced a \emph{learn-to-improve (L2I)} pipeline, which starts from an initial permutation and repeatedly applies local modifications to improve it.

Both L2C and L2I are inherently iterative: they start from an empty/partial solution or an initial solution
and improve it step by step.
This makes inference naturally less efficient and less friendly to large instances, and robustness to outliers is achieved through relatively complex neural architectures and nontrivial inference procedures.

Our general framework applies to QAP in a single-pass direct-output manner: we produce
assignment logits once and deterministically output a permutation without repeated optimization loops, yielding an
immediate speed advantage.
Moreover, with only a lightweight Sinkhorn normalization layer added to the base framework, we obtain outlier-friendly
behavior at very low additional cost.

\subsection{Baselines}
We compare against two classes of baselines that cover both classical matching solvers and representative neural methods on QAPLIB.

\textbf{(i) Classical solvers (non-learning).}
We include \textbf{SM}, \textbf{RRWM}, and \textbf{SK-JA} \cite{kushinsky2019sinkhorn} as standard non-learning baselines for graph matching / QAP.

\textbf{(ii) Neural methods.}
We evaluate \textbf{NGM}, \textbf{RGM}, and \textbf{SAWT} as representative learning-based baselines on QAPLIB.

We follow the same per-class reporting protocol as prior QAPLIB evaluations, and report the mean/min/max optimality
gaps for each QAPLIB class together with the average per-instance inference time.
As the public releases of both $\textbf{SAWT}$ and $\textbf{RGM}$ do not include a complete evaluation setup and code for QAPLIB, we are unable to reproduce their experiments locally. Instead, we directly cite their reported timing results. Consequently, due to hardware disparities, these runtime comparisons should be interpreted as indicative references rather than strictly controlled benchmarks. We also note an additional caveat in their QAPLIB results: for the WIL class they report multiple results, whereas within range ($n\in[12,64]$) where they reported, WIL  contains only a single instance (WIL50). Since we cannot verify the exact evaluation protocol (e.g., whether additional variants or sizes were included), we therefore quote their reported numbers in full without further reinterpretation.

\subsection{Effect of the Frank--Wolfe Iteration Budget}
\label{app:qap-fw-budget}

We evaluate the sensitivity of QAP solution quality and
runtime to the Frank--Wolfe decomposition budget ($T$) on QAPLIB.
Table~\ref{tab:qap-fw-quality-time} reports relative quality
normalized by the weighted solution obtained at $T=10000$,
together with runtime for each budget.

Solution quality improves substantially as the budget
increases from $T=50$ to $T=1000$, rising from 94.247\%
to 99.237\% of the reference quality. Further iterations
yield progressively smaller improvements. At our chosen
budget of $T=3000$, relative quality reaches 99.837\%,
with a runtime of 4.410 seconds. Increasing the budget
to $T=10000$ adds only 0.163 percentage points of relative
quality while increasing runtime to 15.798 seconds,
approximately $3.6\times$ higher. We therefore use
$T=3000$ to balance solution quality and runtime.

Compared with Maximum Coverage, QAP empirically benefits
from a substantially larger decomposition budget.
The approximation guarantee in Theorem~\ref{thm:FWextension} provides
qualitative context: the squared approximation error
is bounded by $4D^2/(T+1)$, where $D$ is the diameter
of the feasible polytope. Our choice of $T$ is guided
by this and the observed quality--runtime trade-off.

\begin{table}[h]
\centering
\caption{QAP solution quality under different FW iteration budgets on QAPLIB. Results are normalized
by the weighted solution obtained at T = 10000. Time recorded on an NVIDIA GeForce RTX 4090 GPU and not directly comparable with other results.}
\label{tab:qap-fw-quality-time}
\setlength{\tabcolsep}{8pt}
\begin{tabular}{rcc}
\toprule
FW budget $T$ & Relative quality (\%) & Time (s) \\
\midrule
   50 & 94.247 & 0.054 \\
  100 & 95.632 & 0.105 \\
  200 & 96.825 & 0.216 \\
  500 & 98.595 & 0.594 \\
 1000 & 99.237 & 1.304 \\
 2000 & 99.672 & 2.828 \\
 \textbf{3000} & \textbf{99.837} & \textbf{4.410} \\
 4000 & 99.893 & 6.017 \\
 5000 & 99.898 & 7.639 \\
 7500 & 99.921 & 11.707 \\
10000 & 100.000 & 15.798 \\
\bottomrule
\end{tabular}
\end{table}
\FloatBarrier

\section{Details and Further Experiments on TSP}

\subsection{Experiment Setup}
\label{sec:tsp-setup}
\paragraph{Instance distribution and data splits.}
We train and evaluate on the 2D Euclidean TSP, where each instance consists of $n$ points in $[0,1]^2$ and tour length is computed under standard Euclidean distances.
We use the training and evaluation data from \textbf{COExpander} \cite{Ma2025COExpanderAS}. During training, for a given $n$, at each epoch, 64 instances are sampled from the training pool, and different data are used across epochs.
The test split has 1280 instances for TSP-50 and TSP-100, 128 instances for TSP-500 and 32 instances for TSP-1000, so that most methods could complete the evaluation in reasonable time frame. 

\paragraph{Graph construction and input features.}
For each instance we construct the complete graph on $n$ nodes (undirected edges enumerated with $u<v$), and represent it in PyG using the standard bidirected encoding (each undirected edge appears in both directions in \texttt{edge\_index}).
Node features are the raw coordinates augmented with a Fourier positional encoding (with $F{=}4$ frequencies) passed through a small MLP before message passing.
Edge features consist of a single scalar channel: the LP relaxation value $x_{\text{LP}}(e)$ for the undirected edge, duplicated across both directed copies.

\paragraph{LP relaxation feature computation.}
We compute $x_{\text{LP}}$ by solving the subtour-elimination LP using the QSopt LP solver, which is also the LP solver used by Concorde.
Edge costs are obtained from Euclidean distances and scaled to integer weights using a fixed coordinate scale of $10^6$ for the QSopt interface.
LP solves are performed on CPU for every instance in both train and test splits.

\paragraph{Model, optimizer, and training schedule.}
The encoder is a GATv2 model with residual blocks: hidden size 128, 3 message-passing layers, and 4 attention heads per layer.
We train with AdamW  with a learning rate of $10^{-3}$ and a weight decay of $10^{-4}$ for 50 epochs, using a single global RNG seed (0) for reproducibility. The best model is saved and used for inference.

\paragraph{Decomposition hyperparameters used in training and inference.}
We use a fixed decomposition budget of $T=64$ iterations in both training and testing, and a max-entropy regularization with temperature \texttt{fw\_entropy\_tau}$=2.0$ is added to the decomposition.
All reported TSP test results use the same hyperparameters as training and the reported result is the best permutation from the decomposition.

\subsection{Related Work: Neural Methods for TSP}

\paragraph{Learning paradigms: LC vs.\ GP.}
Neural approaches for TSP are often categorized by the granularity at which they make decisions. According to \cite{Ma2025COExpanderAS}, the \emph{Learning-to-Construct (LC)} paradigm builds a tour sequentially, repeatedly selecting the next decision conditioned on the current partial solution.
A key strength of LC is that feasibility can be \emph{hard-enforced} during decoding (e.g., by imposed masks that rule out invalid moves), which provides a transparent and reliable path to valid tours.
However, this same sequential structure typically incurs high inference latency on large instances, since the model must execute many dependent steps.

In contrast, the \emph{Global Prediction (GP)} paradigm predicts a globally complete structure in one (or a few) forward passes, commonly as an edge-probability heatmap, and then applies a separate recovery procedure to obtain an integer-feasible tour.
GP can be substantially faster at inference and has become a dominant template for large-scale settings, including diffusion- and consistency-based solvers that generate global solutions through iterative denoising.
At the same time, GP shifts difficulty into the heatmap-to-solution stage: the final performance can depend heavily on the specific decoding/recovery algorithm (e.g., Parallel Sampling, Greedy, Monte Carlo Tree Search (MCTS)) and additional improvement heuristics (e.g., 2-opt), making the pipeline sensitive to post-processing choices and sometimes blurring what is being evaluated (the neural predictor vs.\ the downstream solver), since heatmap recovery is itself a combinatorial optimization subroutine.

\paragraph{Adaptive Expansion (AE) and the remaining heatmap bottleneck.}
\textbf{COExpander} introduces \emph{Adaptive Expansion (AE)} as an intermediate paradigm intended to combine the best of both worlds:
it uses a global predictor to produce a heatmap under partial-solution prompting, and then expands the determined set via a determination/decoding operator, effectively adapting the decision granularity across iterations.
While AE improves the efficiency--feasibility trade-off relative to pure LC or pure GP, it still inherits a core limitation of GP-style pipelines: it relies on predicting a heatmap and then recovering an integer tour via post-processing. So overall success remains coupled to the robustness and tuning of heatmap-based recovery.

\paragraph{Toward end-to-end pipelines with reduced problem-specific setup.}
 \citet{xia2024position}\ critically examine the prevalent \emph{heatmap-guided post-hoc search} paradigm for large-scale TSP (notably heatmap-guided MCTS), questioning the practical value of learning heatmaps that ultimately require substantial downstream search and advocating a shift toward more \emph{autonomous and generalizable} ML pipelines with less hand-crafted machinery and stronger guarantees.
Motivated by this perspective, our method aims to retain LC's feasibility discipline and GP's efficiency while avoiding the heatmap bottleneck:
we adopt an end-to-end pipeline that differentiates through the selected FW branch, with the \emph{optimization objective aligned with the prediction target}---we do not learn a surrogate heatmap and then rely on a separate recovery/search step, but instead directly optimize for improved feasible tours through a structured, deterministic procedure.
This design reduces reliance on hand-crafted post-processing, improves robustness across instance distributions, and enables stronger theoretical characterization of the resulting algorithmic guarantees.

\subsection{Baselines}
We compare against four classes of baselines that span both classical optimization and representative neural paradigms under different supervision regimes.

\textbf{(i) Exact solvers (non-learning).}
We report \textbf{Concorde} as an exact TSP solver, and the \textbf{Gurobi} optimizer as an exact baseline on small instances.
For larger instances, exact results are omitted when the solver fails to finish within the runtime cutoff.

\textbf{(ii) Classical heuristics (non-learning).}
We include \textbf{LKH3} \cite{helsgaun2017extension} as a strong hand-engineered heuristic, run for 500 trials.

\textbf{(iii) Neural LC baselines.}
We evaluate \textbf{RL4CO (SymNCO)} \cite{berto2025rl4co} as a representative \emph{Learning-to-Construct (LC)} method trained with reinforcement learning.
LC-style solvers typically decode solutions sequentially and rely on imposed masks (or equivalent constraints) to hard-guarantee feasibility during construction.

\textbf{(iv) Neural GP/AE baselines.}
We test heatmap-driven pipelines in both RL-style and supervised regimes.
\textbf{DIMES} \cite{qiu2022dimes} is included as a representative GP-style method trained via meta-RL.
On the supervised side, we include \textbf{DIFUSCO} \cite{sun2023difusco}, a diffusion-based GP approach, and \textbf{COExpander}, a heatmap-based \emph{Adaptive Expansion (AE)} method that achieves state-of-the-art performance.

Overall, under a fixed compute budget, we aim to test the landscape as comprehensively as possible across (a) paradigms (LC vs.\ GP vs.\ AE) and (b) supervision regimes (RL / meta-RL vs.\ supervised learning), while transparently reporting any omissions caused by solver timeouts or training-cost constraints.

\subsection{Detailed Results}
\begin{table*}[t]
\centering
\caption{Raw performance comparison on TSP instances. We report the objective value (Obj.) and inference time (s).}
\resizebox{\textwidth}{!}{
\begin{tabular}{l|c|cc|cc|cc|cc}
\hline
\multirow{2}{*}{Method} & \multirow{2}{*}{Type} & \multicolumn{2}{c|}{TSP-50} & \multicolumn{2}{c|}{TSP-100} & \multicolumn{2}{c|}{TSP-500} & \multicolumn{2}{c}{TSP-1000} \\
\cline{3-10}
 & & Obj. & Time & Obj. & Time & Obj. & Time & Obj. & Time \\
\hline
Concorde (Optimal) & Exact & 5.69 & 0.04 & 7.76 & 0.21 & 16.55 & 17.91 & 23.12 & 435.61 \\
Gurobi (Optimal)$^*$ & Exact & 5.69 & 0.27 & 7.76 & 1.14 & - & - & - & - \\
LKH3 (500) & Heuristics & 5.69 & 0.03 & 7.76 & 0.05 & 16.70 & 0.32 & 23.41 & 1.06 \\
\hline
DIMES (RL+S) & RL & 6.47 & 0.05 & 8.66 & 0.08 & 19.09 & 0.39 & 26.35 & 0.79 \\
RL4CO (Sym-NCO)$^\dagger$ & RL & 5.76 & 0.36 & 8.83 & 0.66 & - & - & - & - \\
DIFUSCO (S=1, I=50) & SL & 5.72 & 0.43 & 7.86 & 0.66 & 18.17 & 2.19 & 25.64 & 7.68 \\
COExpander (S=1,Ds=3,Is=5) & SL & 5.69 & 0.11 & 7.76 & 0.20 & 17.17 & 0.69 & 24.63 & 2.50 \\
\hline
\textbf{Ours-lp} & UL & 5.79 & 0.23 & 8.03 & 0.26 & 17.76 & 1.27 & 25.03 & 4.26 \\
\textbf{Ours-lp+learned matching} & UL & 5.77 & 0.52 & 7.98 & 0.74 & 17.70 & 2.77 & 24.84 & 9.16 \\
\hline
\end{tabular}
}
\begin{flushleft}
\footnotesize
$^*$ Results on larger instances are omitted since it was unable to finish within the cutoff time.\\
$^\dagger$ Results are omitted due to prohibitive computational costs during training on large datasets.
\end{flushleft}

\label{tab:tsp_results_full}
\end{table*}

\paragraph{Runtime breakdown.}
In \cref{tab:time_consumption}, we measure the average per-instance wall-clock cost of each stage in the TSP pipeline . The Subtour-Elimination LP and \textsc{FW Decomposition} dominates the runtime; in the learned matching version, the matching takes up a large portion of the overall runtime.

\begin{table}[!ht]
\centering
\begin{tabular}{lc}
\toprule
\textbf{Stage} & \textbf{Time (s/inst)} \\
\midrule
LP (Subtour-Elimination) & 3.09 \\
Neural Network       & 0.06 \\
FW over trees            & 0.93 \\
Greedy matching & 0.18 \\
Learned Matching & 4.23\\
\bottomrule
\end{tabular}
\caption{Average per-instance time on TSP1000.}
\label{tab:time_consumption}
\end{table}

\paragraph{Generalization Experiments (Across Instance Sizes).}
We further evaluate how well a model trained on one problem size transfers to other sizes.
Specifically, we train three separate models on TSP-100, TSP-500, and TSP-1000, respectively, and test each model on the same evaluation splits at sizes $\{100,500,1000\}$.
Table~\ref{tab:generalization} reports the resulting tour lengths.

Overall, we observe strong cross-size generalization: performance is largely stable when transferring between sizes, with only modest degradation when training on smaller instances and testing on larger ones.
For example, a model trained on TSP-100 remains competitive on TSP-500/1000, albeit with a small gap relative to models trained on the target size.
Conversely, models trained on larger instances transfer well to smaller instances, with only minor differences on TSP-100.
These results suggest that the learned scoring function captures size-agnostic geometric regularities of Euclidean TSP, enabling effective transfer across instance scales without retraining on each target size.
\begin{table}[H]
    \centering
    \caption{Generalization performance across different instance sizes.}
    \label{tab:generalization}
    \begin{tabular}{l|ccc}
        \toprule
        \multirow{2}{*}{Training Set} & \multicolumn{3}{c}{Testing Set} \\
        \cmidrule(lr){2-4}
                  & TSP-100 & TSP-500 & TSP-1000 \\
        \midrule
        TSP-100   & 8.03    & 17.85   & 25.28    \\
        TSP-500   & 8.09    & 17.76   & 25.06    \\
        TSP-1000  & 8.16    & 17.81   & 25.03    \\
        \bottomrule
    \end{tabular}
\end{table}

\subsection{Discussion on Learned Matching}
\label{app:learned-matching}

\paragraph{Motivation and positioning.}
In our TSP pipeline, each candidate spanning tree produced by the FW decomposition is rounded into a tour via the Christofides procedure.
A key step in Christofides is computing a perfect matching on the odd-degree vertex set induced by the tree.
Our default implementation uses a greedy matching routine for efficiency, while the learned-matching variant replaces this step with a trainable module that is \emph{optimized end-to-end for the TSP objective}, rather than supervised to recover an exact minimum-weight perfect matching.

\paragraph{Matching as another structured polytope component.}
For a given tree, let $O$ be its odd-degree vertex set (with $|O|$ even), and let $E_O$ denote all unordered pairs in $O$, which forms a structured polytope.
The matching network outputs a vector $S \in (0,1)^{|E_O|}$ over $E_O$.
We then run our \textsc{FW Decomposition} procedure on the matching polytope over $O$ to obtain a short convex decomposition
$A_{\text{match}} \approx \sum_{k} \beta_k \, {M_k}$, where each $M_k$ is an \emph{integral} perfect matching (an extreme point of the polytope),
$\beta_k \ge 0$, and $\sum_k \beta_k = 1$.

\paragraph{Joint objective (learning matchings for tours).}
Training couples the spanning-tree and matching components through the downstream tour length.
Concretely, the TSP model produces $X$ which is then decomposed through \textsc{FW Decomposition} over the spanning-tree polytope, and that yields a short list of trees with weights $\{\alpha_t\}$.
For each tree $t$, we form its odd set $O_t$, which then goes through the neural network which outputs $S_t$ on $E_{O_t}$, and we follow that by applying FW to obtain matchings $\{M_{t,k}\}$ with weights $\{\beta_{t,k}\}$.
We evaluate the Christofides tour length for each $(t,k)$ pair and minimize the mixture objective
\[
\mathcal{L}_{\text{tour}}
= \sum_{t} \alpha_t \sum_{k} \beta_{t,k}\; \!\left(\textsc{Christofides}(T_t, M_{t,k})\right),
\]
where $ \textsc{Christofides}(T_t, M_{t,k}) $ represents the length of the TSP tour generated by the given parameters. So the learned matchings are directly trained to be \emph{useful for producing short tours} under our pipeline.
We also include reconstruction loss in the same way as with other problems, in order to keep X close or inside the polytope. This training also does not require ground-truth matchings; gradients are driven by the downstream TSP objective through the Christofides rounding.

\paragraph{Matching network architecture.}
The learned matching utilizes a lightweight 2 layer MLP structure to perform fast inferences, it encodes nodes using Fourier positional encoding of the 2D coordinates as input features. The matching \textsc{FW Extension} uses a small iteration budget $T_{\text{match}}=8$. The TSP neural network working on the spanning tree polytope follows the same setup as \cref{sec:tsp-setup}.

Empirically, the learned matching variant improves objective values over the greedy-matching version (\cref{tab:tsp_results_full}) at the cost of higher inference time,
and the runtime breakdown (\cref{tab:time_consumption}) shows that matching can dominate the additional overhead on large instances.

\subsection{Ablation Study and Discussion on Projection}

Following the experiments for Maximum Coverage, we ablate the effect of projecting the output of the neural network back into the feasible polytope. Table~\ref{tab:ablation_proj} shows that, for TSP, removing projection yields slightly better performance on both TSP-500 and TSP-1000.

We attribute this to a structural property of (2D) Euclidean TSP.
Unlike problems with highly heterogeneous instance (e.g., QAP), Euclidean TSP is constrained by a low-dimensional geometric embedding and metric distances, which makes extremely ``hard'' out-of-distribution instances comparatively rare—both for synthetic generators and for real-world benchmarks that remain within the same problem family.
To support this claim empirically, we further test the model trained on TSP-1000 on 2D Euclidean instances from \textbf{TSPLIB}~\cite{reinelt1991tsplib} with 51 to 1002 nodes.
As shown in \cref{tab:tsplib_tsp1000_model}, our method remains robust even on this substantially broader collection.

These results suggest that, for problems like Euclidean TSP, explicit projection is not only unnecessary but can be counterproductive, while the no-projection design is fully viable and can even yield better performance in practice.

\begin{table}[!ht]
    \centering
    \caption{Ablation results with and without projection}
    \label{tab:ablation_proj}
    \begin{tabular}{lcc}
        \toprule
        Dataset & With Projection & No Projection \\
        \midrule
        TSP-500  & 18.06 & 17.76 \\
        TSP-1000 & 25.47 & 25.03 \\
        \bottomrule
    \end{tabular}
\end{table}

\begin{table}[H]
\centering
\caption{Performance on TSPLIB of the model trained on TSP-1000.}
\label{tab:tsplib_tsp1000_model}
\begin{tabular}{l l r r r r}
\toprule
Train Set & Test Set & Avg. Pred & Avg. Opt & Gap & Avg. Time (s) \\
\midrule
TSP-1000 & TSPLIB & 8.848170 & 8.061844 & 9.75\% & 1.703 \\
\bottomrule
\end{tabular}
\end{table}

\subsection{Ablation Study: LP + Noise }
\label{sec:lp+noise}
To test whether the gains come from \emph{learned} structure rather than arbitrary perturbations of the LP relaxation, we construct a randomized baseline by adding i.i.d.\ Gaussian noise to the subtour-LP solution: $\tilde{x}=x_{\mathrm{LP}}+\epsilon$ with $\epsilon\sim\mathcal{N}(0,\sigma^2)$ and $\sigma=0.2$.
We then run the \emph{same} FW decomposition and downstream rounding pipeline using $\tilde{x}$ as the input weights.
For each instance we repeat this procedure for 10 independent noise draws and report the \emph{best} (shortest) tour length, which is a favorable ``best-of-10'' setting for the noise baseline.
As shown in Table~\ref{tab:ablation_lp_noise}, this randomized perturbation is unable to match our method, with the gap becoming especially pronounced on larger instances (TSP500--TSP1000).
This indicates that our network learns instance-dependent signals that improve the pipeline beyond what can be obtained by injecting random noise into $x_{\mathrm{LP}}$.

\begin{table}[!ht]
\centering
\small
\begin{tabular}{lcc}
\toprule
 & \textbf{LP + Gaussian noise+FW} ($\sigma{=}0.2$, best of 10) & \textbf{Ours} \\
\midrule
TSP50   & 5.93  & 5.79 \\
TSP100  & 8.86  & 8.03 \\
TSP500  & 32.83 & 17.76 \\
TSP1000 & 65.51 & 25.03 \\
\bottomrule
\end{tabular}
\caption{Ablation results on replacing the learned edge weights with random perturbations of the LP solution. Lower is better.}
\label{tab:ablation_lp_noise}
\end{table}

\FloatBarrier
\section{Further Discussion on Inference Settings}
\label{sec:inference_settings}

In the field of Neural Combinatorial Optimization (NCO), strategies to refine solution quality \emph{during inference} are widely adopted. While terminology varies across domains—often referred to as Test-Time Optimization (TTO) or Local Search/Sampling—these procedures share a common goal: to improve the solution beyond a single neural forward pass. Broadly, they fall into three categories:
\begin{enumerate}
    \item \textbf{Gradient-based Active Search:} Fine-tuning model parameters on individual test instances (e.g., active search).
    \item \textbf{Search-based Refinement:} Applying discrete search algorithms, such as beam search, MCTS, or heuristic local search (e.g., 2-opt), to the model's output.
    \item \textbf{Parallel Sampling:} Running multiple inference trajectories, such as multiple diffusion noise seeds, batched augmentations, or independent neural forward passes, to enlarge the candidate pool and select the best solution. We distinguish this from native candidate-set decoding, where a single neural forward pass deterministically or procedurally produces several candidates as part of the method's prescribed decoder.
\end{enumerate}
\paragraph{Decoupling Learned Priors from Search Strategies.}
While these inference-time techniques significantly enhance performance, they introduce a critical challenge in evaluation: \textbf{they often obscure the intrinsic quality of the learned neural model}. 

Recent studies have critically examined this phenomenon. \cite{bother2022s} demonstrated that in tree-search-based neural methods, the search algorithm frequently performs the majority of the "heavy lifting," rendering the contribution of the learned neural guidance marginal or, in some cases, redundant. \cite{xia2024position} finds that this is also true for tree-search based heatmap recovery. They further point out that this inconsistency is particularly concerning because the search algorithm, which is critical for determining the final solution, is not integrated during neural network training. This confounding effect where the solver relies more on inference time compute than on the learned structural prior has also been noted in recent works such as \textbf{UCOM2} \cite{bu2024tackling} and \textbf{GeoNCO} \cite{karalias2025geometric}, which sets a clear difference between TTO and non-TTO methods and make no comparison between them.

\paragraph{Evaluation Protocol: The "One-Shot" Setting.}
In this work, our primary goal is to evaluate the intrinsic constructive capability of the proposed framework. We aim to assess how well the model aligns with the problem structure in a minimal-latency setting, without relying on extensive post-hoc search. Consequently, we strictly standardize the evaluation to a \textbf{non-TTO} context:
\begin{itemize}
    \item For Maximum Coverage baselines (e.g., \textbf{CardNN}, \textbf{UCOM2}, \textbf{GeoNCO}), we utilize their "short" inference modes which comes without gradient-based finetune on inference or iterative refinement steps.
    \item For TSP baselines (e.g., \textbf{DIMES}, \textbf{DIFUSCO}, \textbf{COExpander}), which often rely on active search, parallelized diffusion steps or local search to refine results, we restrict them to their greedy or native candidate-set decoding equivalents. We exclude parallel sampling strategies
    that enlarge the candidate pool through repeated inference, such as running diffusion-based solvers with multiple noise seeds. However, when a method's native decoder generates multiple feasible candidates from one neural network pass and selects the best one, we still regard it as one-shot. For example, \textbf{DIMES (RL+S)} uses a number of samples generated from a single neural network pass, so this is treated as native candidate-set decoding rather than parallel sampling. Our FW decomposition follows the same principle: one neural output is decomposed into a sparse set of feasible vertices, and inference reports the best vertex among them.
\end{itemize}
By isolating the raw performance of the underlying neural architectures, we provide a clearer benchmark of their structural learning capabilities, ensuring that improvements are attributed to better representation learning rather than increased inference computation. This protocol also keeps
the comparison consistent with the decoding budgets used by strong baselines, where candidate-set
selection is often part of the native decoder, while excluding additional inference-time computation
whose main effect is to enlarge the search pool.
\FloatBarrier
\section{GPU Memory Usage}
\label{app:gpu_memory}

We report the GPU memory footprint of FWNCO across all training configurations. Table~\ref{tab:gpu_memory} reports peak CUDA allocated memory during training. These measurements describe allocated GPU memory,
rather than total device-memory consumption, which can
additionally include reserved but unused memory and CUDA
runtime overhead.

\begin{table}[h]
    \centering
    \caption{Peak CUDA allocated memory during FWNCO training.}
    \label{tab:gpu_memory}
    \begin{tabular}{llr}
        \toprule
        Problem & Dataset / $k$ & Peak memory (GiB) \\
        \midrule
        MC  & Random500, $k=10$   & 0.135 \\
        MC  & Random500, $k=50$   & 0.135 \\
        MC  & Random1000, $k=20$  & 0.194 \\
        MC  & Random1000, $k=100$ & 0.194 \\
        \midrule
        TSP & TSP50               & 0.043 \\
        TSP & TSP100              & 0.106 \\
        TSP & TSP500              & 2.012 \\
        TSP & TSP1000             & 7.708 \\
        \midrule
        QAP & QAP32               & 5.168 \\
        \bottomrule
    \end{tabular}
\end{table}

Overall, the framework has a relatively modest memory footprint of below 8\,GiB for every configuration. These results indicate that FWNCO does not require the latest high-end GPUs for the evaluated configurations and suggest potential for scaling to substantially larger instances.
\FloatBarrier
\section{Polytopes with Efficient (Approximate) LMO}
\label{sec:cases-apendix}

\paragraph{Matroid polytope} Matroid polytopes provide a unified geometric framework for a broad class of combinatorial constraints. Given a matroid $\mc{M} = (V,\mc{I})$ with rank function $r(\cdot)$, Edmonds’ classical formulation characterizes the associated matroid base polytope as
\[
\mc{P}_{\text{base}} = \left\{ \bx \in \R^n_{\ge 0} \;\middle|\; \sum_{i \in A} x_i \le r(A)\ \forall A \subseteq V,\ \sum_{i \in V} x_i = r(V) \right\}.
\]
A key property of this polytope is that it admits an efficient LMO: optimizing a linear function over $\mc{P}_{\text{base}}$ reduces to finding a maximum-weight base of the matroid, which can be solved in polynomial time via the greedy algorithm. This class of polytopes subsumes several fundamental constraints as special cases:
\begin{itemize}
    \item \textbf{Cardinality constraints (uniform matroid).}  
    The ground set is $V=\{1,\dots,n\}$ with rank function $r(A)=\min\{|A|,k\}$. The corresponding base polytope is the hypersimplex
    $
    \Delta_{n,k} = \left\{ \bx \in \R^n_{\ge 0} \;\middle|\; \sum_{i=1}^n x_i = k,\; 0 \le x_i \le 1 \right\}.
    $

    \item \textbf{Partition matroids.}  
    Let $V = \bigcup_{j=1}^c V_j$ be a partition of the ground set with associated budgets $k_1,\dots,k_c$. The rank function is $r(A)=\sum_{j=1}^c \min\{|A \cap V_j|,k_j\}$, and the base polytope is given by
    $
    \left\{ \bx \in \R^n_{\ge 0} \;\middle|\; \sum_{i \in V_j} x_i = k_j \ \forall j \right\}.
    $

    \item \textbf{Spanning tree constraints (graphic matroid).}  
    Given an undirected graph $G=(V,E)$, the ground set consists of edges and the rank function is $r(A)=|V|-\kappa(A)$, where $\kappa(A)$ denotes the number of connected components in $(V,A)$. The corresponding base polytope is
    $
    \left\{ \bx \in \R^{|E|}_{\ge 0} \;\middle|\; \sum_{e \in E} x_e = |V|-1,\; \sum_{e \in A} x_e \le |V|-\kappa(A)\ \forall A \subseteq E \right\}.
    $
\end{itemize}


Our setup naturally covers all of these cases within a single algorithmic framework. In contrast, \citet{karalias2025geometric} study these settings in a case-specific manner, requiring the design of separate projection operators onto each polytope as well as bespoke decomposition procedures tailored to each constraint. By relying solely on the availability of an efficient LMO, our results apply uniformly across all these matroidal constraints without modifying the algorithm.

\paragraph{Birkhoff polytope}
\label{sec:poly-assignment}
An $n \times n$ matrix is called \emph{doubly stochastic} if all of its entries are nonnegative and the sum of the entries in each row and each column equals one. \emph{Permutation matrices} form a special subclass of doubly stochastic matrices, characterized by having exactly one entry equal to one in each row and column and zeros elsewhere. The set of all $n \times n$ doubly stochastic matrices constitutes a convex polytope known as the \emph{Birkhoff polytope} $\mathcal{B}_n$. The Birkhoff polytope lies in an $(n - 1)^2$-dimensional affine subspace of $n^2$-dimensional Euclidean space defined by $2n - 1$ independent linear constraints specifying that the row and column sums all equal 1. The extreme points of Birkhoff polytope $\mathcal{B}_n$ are exactly the permutation matrices. Given a residual matrix $R \in \mathbb{R}^{n\times n}$, the LMO over $\mathcal{B}_n$ solves
\[
P^\star
\;\in\;
\argmax_{P \in \ext(\mathcal{B}_n)} \langle R, P \rangle=\sum_i\sum_j R(i,j)P(i,j)
\]
This problem is equivalent to finding a maximum-weight perfect matching in a complete bipartite graph $G=(U,V,E)$, where $U$ and $V$ index the rows and columns of $R$, respectively, and each edge $(i,j)\in E$ has weight $R(i,j)$. Hence the linear maximization problem over Birkhoff polytope can be solved using the Hungarian algorithm in time $O(n^3)$. Hence, our approach subsumes the setting considered in \cite{nerem2025differentiable} without requiring any projection onto the Birkhoff polytope. In contrast, the authors of \cite{nerem2025differentiable} develop a specialized decomposition algorithm and introduce an additional loss term to penalize neural network outputs that lie outside the polytope.

\paragraph{DMO: greedy bipartite matching.}
Note that a simple greedy algorithm produces a 0.5-approximation to
maximum weight matching for any (positively) weighted graph. This still outputs a rather good matching while being significantly faster than running an exact Hungarian algorithm oracle at every iteration. Empirically, we observe that the greedy DMO does not sacrifice quality.

\paragraph{Matching polytope} Building on the discussion of the Birkhoff polytope, our approach also applies to matching and perfect matching polytopes.

\FloatBarrier


\end{document}